\documentclass[runningheads]{llncs}

\usepackage{amsmath}
\usepackage{enumerate}
\usepackage{fancyhdr}
\usepackage{color}
\usepackage{listings}
\usepackage{graphicx}
\usepackage{amssymb}
\usepackage{mathtools}
\usepackage{cite}

\usepackage{lipsum}

\usepackage[subtle]{savetrees}

\usepackage{algorithm}
\usepackage{algorithmic}

\usepackage{mathnotations}

\usepackage[misc]{ifsym}

\newtheorem{assumption}{Assumption}

\newcommand{\E}{\mathbb{E}}
\newcommand{\x}{\Tilde{x}_t}

\begin{document}
\title{Model Selection in Reinforcement Learning with General Function Approximations}
\titlerunning{Model Selection for RL}

 \author{Avishek Ghosh$^\star$\inst{1}(\Letter) \and
 Sayak Ray Chowdhury$^\star$\inst{2} 
 %\and
 \authorrunning{A. Ghosh and S. Ray Chowdhury}
\institute{Halıcıoğlu Data Science Institute (HDSI), UC San Diego, USA \and
Boston University, USA \\
 \email{a2ghosh@ucsd.edu, sayak@bu.edu}\\
}
}

\toctitle{Model Selection in Reinforcement Learning}
\tocauthor{Avishek~Ghosh, Sayak~Ray Chowdhury}
\maketitle              % typeset the header of the contribution
\vspace{-5mm}
\begin{abstract}
We\footnote{$^\star$ Avishek Ghosh and Sayak Ray Chowdhury contributed equally.} consider model selection for classic Reinforcement Learning (RL) environments -- Multi Armed Bandits (MABs) and  Markov Decision Processes (MDPs) -- under general function approximations. In the model selection framework, we do not know the function classes, denoted by $\mathcal{F}$ and $\mathcal{M}$, where the true models -- reward generating function for MABs and and transition kernel for MDPs -- lie, respectively. Instead, we are given $M$ nested function (hypothesis) classes such that true models are contained in at-least one such class. In this paper, we propose and analyze  efficient model selection algorithms for MABs and MDPs, that \emph{adapt} to the smallest function class (among the nested $M$ classes) containing the true underlying model. Under a separability assumption on the nested hypothesis classes, we show that the cumulative regret of our adaptive algorithms match to that of an oracle which knows the correct function classes (i.e., $\cF$ and $\cM$) a priori. Furthermore, for both the settings, we show that the cost of model selection is an additive term in the regret having weak (logarithmic) dependence on the learning horizon $T$.
\end{abstract}

\section{INTRODUCTION}
\label{sec:intro}
We study the problem of  \emph{model selection} for Reinforcement Learning problems, which refers to choosing the appropriate hypothesis class, to model the mapping from actions to expected rewards. We choose two particular frameworks---(a) Multi-Armaed Bandits (MAB) and (b) markov Decision Processes (MDP). Specifically, we are interested in studying the model selection problems for these frameworks without any function approximations (like linear, generalized linear etc.). Note that, the problem of model selection plays an important role in applications such as personalized recommendations, autonomous driving, robotics as we explain in the sequel. Formally, a family of nested hypothesis classes $\mathcal{H}_f$, $f \in \mathcal{F}$ is specified, where each class posits a plausible model for mapping actions to expected rewards. Furthermore, the family $\mathcal{F}$ is totally ordered, i.e., if $f_1 \leq f_2$, then $\mathcal{H}_{f_1}\subseteq \mathcal{H}_{f_2}$. It is assumed that the true model is contained in at least one of these specified families. Model selection guarantees then refer to algorithms whose regret scales in the complexity of the \emph{smallest hypothesis class containing the true model}, even though the algorithm was not aware apriori.

Multi-Armed Bandits (MAB) \cite{cesa2006prediction} and Markov decision processes (MDP) \cite{puterman2014markov} are classical frameworks to model a reinforcement learning (RL) environment, where an agent interacts with the environment by taking successive decisions and observe rewards generated by those decisions. One of the objectives in RL is to maximize the total reward accumulated over multiple rounds, or equivalently minimize the \emph{regret} in comparison with an optimal policy \cite{cesa2006prediction}. Regret minimization is useful in several sequential decision-making problems such as portfolio allocation and sequential investment, dynamic resource allocation in communication systems, recommendation systems, etc. In these settings, there is
no separate budget to purely explore the
unknown environment; rather, exploration and exploitation need to be carefully balanced. 

Optimization over large domains under restricted feedback is an important problem and has found applications in dynamic pricing for economic markets \cite{BesZee09:dynpricinglearning}, wireless communication \cite{chang-lai-1987} and recommendation platforms (such as Netflix, Amazon Prime). Furthermore, in many applications (e.g., robotics, autonomous driving), the number of actions and the observable number of states can be very large or even infinite, which makes RL challenging, particularly in generalizing learnt
knowledge across unseen states and actions. For example, the game of Go has a state space with size $3^{361}$, and the state and action spaces of certain robotics applications can even be continuous.
In recent years, we have witnessed an explosion in the RL literature to tackle this challenge, both in theory (see, e.g., \cite{osband2014model,chowdhury2019online,ayoub2020model,wang2020provably,kakade2020information}), and in practice (see, e.g., \cite{mnih2013playing,williams2017model}).

In the first part of the paper, we focus on learning an unknown function $f^* \in \mathcal{F}$, supported over a compact domain, via online noisy observations. If the function class $\mathcal{F}$ is known, the optimistic algorithm of \cite{russo2013eluder} learns $f^*$, yielding a regret that depends on \emph{eluder dimension} (a complexity measure of function classes) of $\mathcal{F}$. However, in the applications mentioned earlier, it is not immediately clear how one estimates $\mathcal{F}$. Naive estimation techniques may yield an unnecessarily big $\mathcal{F}$, and as a consequence, the regret may suffer. On the other hand, if the estimated class, $\hat{\mathcal{F}}$ is such that $\hat{\mathcal{F}} \subset \mathcal{F}$, then the learning algorithm might yield a linear regret because of this infeasibility. Hence, it is important to estimate the function class properly, and here is where the question of model selection appears. The problem of model selection is formally stated as follows---we are given a family of $M$ hypothesis classes $\mathcal{F}_1 \subset \mathcal{F}_2\subset \ldots \subset \mathcal{F}_M$, and the unknown function $f^*$ is assumed to be contained in the family of nested classes. In particular, we assume that $f^*$ lies in $\mathcal{F}_{m^*}$, where $m^*$ is unknown.  Model selection guarantees refer to algorithms whose regret scales in the complexity of the \emph{smallest model class containing the true function $f^*$, i.e., $\mathcal{F}_{m^*}$}, even though the algorithm is not aware of that a priori.

In the second part of the paper, we address the model selection problem for generic MDPs without funcction approximation. The most related work to ours is by \cite{ayoub2020model}, which proposes an algorithm, namely \texttt{UCRL-VTR}, for model-based
RL without any structural assumptions, and it is based on the upper confidence RL and value-targeted regression principles. The regret of \texttt{UCRL-VTR} depends on the \emph{eluder dimension} \cite{russo2013eluder} and the \emph{metric entropy} of the corresponding family of distributions $\cP$ in which the unknown transition model $P^*$ lies. In most practical cases, however, the class $\mathcal{P}$ given to (or estimated by) the RL agent is quite pessimistic; meaning that $P^*$ actually lies in a small subset of $\mathcal{P}$ (e.g., in the game of Go, the learning is possible without the need for visiting all the states \cite{silver2017mastering}). We are given a family of $M$ nested hypothesis classes $\mathcal{P}_1 \subset \mathcal{P}_2 \subset \ldots \subset \mathcal{P}_M $, where each class posits a plausible model class for the underlying RL problem. The true model $P^*$ lies in a model class $\cP_{m*}$, where $m^*$ is unknown apriori. Similar to the functional bandits framework, we propose learning algorithms whose regret depends on the \emph{smallest model class containing the true model $P^*$}.

The problem of model selection have received considerable attention in the last few years. Model selection is well studied in the contextual bandit setting. In this setting, minimax optimal regret guarantees can be obtained by exploiting the structure of the problem along with an eigenvalue assumption \cite{osom, foster_model_selection, ghosh2021problem} We provide a comprehensive list of recent works on bandit model selection in Section~\ref{sec:related_work}. However, to the best of our knowledge, this is the first work to address the model selection question for generic (functional) MAB without imposing any assumptions on the reward structure. 

In the RL framework, the question of model selection has received little attention. In a series of works, \cite{pacchiano2020regret,aldo} consider the corralling framework of \cite{corral} for
 contextual bandits and reinforcement learning. While the corralling framework is versatile, the price for this is that the cost of model selection is multiplicative rather than additive. In particular, for the special case of linear bandits and linear reinforcement learning, the regret scales as $\sqrt{T}$ in time with an additional multiplicative factor of $\sqrt{M}$, while the regret scaling with time is strictly larger than $\sqrt{T}$ in the general contextual bandit. These papers treat all the hypothesis classes as bandit arms, and hence work in a (restricted) partial information setting, and as a consequence explore a lot, yielding worse regret. On the other hand, we consider all $M$ classes at once (full information setting) and do inference, and hence explore less and obtain lower regret.

Very recently, \cite{vidya} study the problem of model selection in RL with function approximation. Similar to the \emph{active-arm elimination} technique employed in standard multi-armed bandit (MAB) problems \cite{eliminate}, the authors eliminate the model classes that are dubbed misspecified, and obtain a regret of $\mathcal{O}(T^{2/3})$. On the other hand, our framework is quite different in the sense that we consider model selection for RL with \emph{general} transition structure. Moreover, our regret scales as $\mathcal{O}(\sqrt{T})$. Note that the model selection guarantees we obtain in the linear MDPs are partly influenced by \cite{ghosh2021problem}, where model selection for linear contextual bandits are discussed. However, there are a couple of subtle differences: (a) for linear contextual framework, one can perform pure exploration, and \cite{ghosh2021problem} crucially leverages that and (b) the contexts in linear contextual framework is assumed to be i.i.d, whereas for linear MDPs, the contexts are implicit and depend on states, actions and transition probabilities.
\vspace{-4mm}
\subsection{Our Contributions}
\vspace{-2mm}
In this paper, our setup considers \emph{any general} model class (for both MAB and MDP settings) that are totally bounded, i.e., for arbitrary precision, the metric entropy is bounded. Note that this encompasses a significantly larger class of environments compared to the problems with function approximation. Assuming nested families of reward function and transition kernels, respectively for MABs and MDPs, we propose adaptive algorithms, namely \emph{Adaptive Bandit Learning} (\texttt{ABL}) and \emph{Adaptive Reinforcement Learning} (\texttt{ARL}). Assuming the hypothesis classes are separated, both \texttt{ABL} and \texttt{ARL} construct a test statistic and thresholds it to identify the correct hypothesis class. We show that these \emph{simple} schemes achieve the regret of $\Tilde{\mathcal{O}}(d^*+\sqrt{d^* \mathbb{M}^* T})$ for MABs and $\Tilde{\mathcal{O}}(d^*H^2+\sqrt{d^* \mathbb{M}^* H^2 T})$ for MDPs (with episode length $H$), where $d_{\mathcal{E}}^*$ is the \emph{eluder dimension} and $\mathbb{M}^*$ is the \emph{metric entropy} corresponding to the smallest model classes containing true models ($f^*$ for MAB and $P^*$ for MDP). The regret bounds show that both \texttt{ABL} and \texttt{ARL} adapts to the true problem complexity, and the cost of model section is only $\mathcal{O}(\log T)$, which is minimal compared to the total regret.
% much less than the total regret of $\Tilde{\mathcal{O}}(d_{\mathcal{E}}^*H^2+\sqrt{d_{\mathcal{E}}^* \mathbb{M}^* H^2 T})$ (see \cite{ayoub2020model}). Hence, the model selection cost is \emph{minimal}.
%In the second part of the paper, we focus on the linear function approximation framework, where the transition kernel is parameterized by a vector $\theta^* \in \R^d$. With norm ($\|\theta^*\|$) and sparsity ($\|\theta^*\|_0$) as problem complexity parameters, we propose two algorithms, namely \emph{Adaptive Reinforcement Learning-Linear (norm)} and \emph{Adaptive Reinforcement Learning-Linear (dim)}, respectively, that adapt to these complexities -- meaning that the regret depends on the actual problem complexities $\|\theta^*\|$ and $\|\theta^*\|_0$. Here also, the costs of model selection are shown to be \emph{minor lower order} terms.

\textbf{Notation} For a positive integer $n$, we denote by $[n]$ the set of integers $\{1,2,\ldots,n\}$. 
% $\|\cdot\|$ and $\inner{\cdot}{\cdot}$ denote $\ell_2$ norm and inner product, respectively, unless otherwise specified. $\gamma_{\min}(A)$ denotes the minimum eigenvalue of the matrix $A$. $\mathbb{B}_d^1$ denotes the unit ball in $\Real^d$ and $\mathbb{S}_+^{d}$ denotes the set of all $d\times d$ positive definite matrices. 
For a set $\cX$ and functions $f,g:\cX \to \Real$, we denote $(f-g)(x):=f(x)-g(x)$ and $(f-g)^2(x):=\left(f(x)-g(x)\right)^2$ for any $x \in \cX$. 
For any $P:\cZ \to \Delta(\cX)$, we denote $(Pf)(z):=\int_{\cX}f(x)P(x|z)dx$ for any $z \in \cZ$, where $\Delta(\cX)$ denotes the set of signed distributions over $\cX$.
% All inequalities between random variables in this paper hold almost surely with
% respect to the underlying probability space.
\vspace{-4mm}
\subsection{Related Work}
\vspace{-2mm}
\label{sec:related_work}
\textbf{Model Selection in Online Learning:} Model selection for bandits are only recently being studied \cite{ghosh2017misspecified,osom}. These works aim to identify whether a given problem instance comes from contextual or standard setting. For linear contextual bandits, with the dimension of the underlying parameter as a complexity measure, \cite{foster_model_selection,ghosh2021problem} propose efficient algorithms that adapts to the \emph{true} dimension of the problem. While \cite{foster_model_selection} obtains a regret of $\mathcal{O}(T^{2/3})$, \cite{ghosh2021problem} obtains a $\mathcal{O}(\sqrt{T})$ regret (however, the regret of \cite{ghosh2021problem} depends on several problem dependent quantities and hence not instance uniform). Later on, these guarantees are extended to the generic contextual bandit problems without linear structure \cite{ghosh2021modelgen,krishnamurthy2021optimal}, where $\mathcal{O}(\sqrt{T})$ regret guarantees are obtained. The algorithm {\ttfamily Corral} was proposed in \cite{corral}, where the optimal algorithm for each model class is casted as an expert, and the forecaster obtains low regret with respect to the best expert (best model class). The generality of this framework has rendered it fruitful in a variety of different settings; see, for example \cite{corral, arora2021corralling}.

% Very recently \cite{vidya} proposes a model selection algorithm in the RL framework with function approximation based on careful elimination schedule; and obtains a regret of $\mathcal{O}(T^{2/3})$.  

\textbf{RL with Function Approximation:} 
Regret minimization in RL under function approximation is first considered in 
\cite{osband2014model}. It makes explicit model-based assumptions and the regret bound depends on the eluder dimensions of the models. In contrast, \cite{yang2019reinforcement} considers a low-rank linear transition model and propose a model-based algorithm with regret $\mathcal{O}(\sqrt{d^3H^3T})$. Another line of work parameterizes the $Q$-\emph{functions} directly, using state-action
feature maps, and develop model-free algorithms with regret $\mathcal{O}(\texttt{poly}(dH)\sqrt{T})$ bypassing the need for fully
learning the transition model \cite{jin2019provably,wang2019optimism,zanette2020frequentist}. A recent line of work \cite{wang2020provably,yang2020provably} generalize these approaches by designing algorithms that work with general and neural function approximations, respectively.

% and prove a similar regret guarantee that depends on the eluder dimension \cite{russo2013eluder} and log-covering number of the
% underlying function class. Recently, \cite{yang2020provably} consider kernel and neural function approximations and designed algorithms with regret characterized by intrinsic complexity of the function classes. 

% Under the assumption that the (action-)value function can be approximated by a linear or a generalized linear function of the feature vectors, these papers develop algorithms with regret bound proportional to $\text{poly}(d)\sqrt{T}$, which is independent of the size of the state and action spaces. 

% Nevertheless, there is a lack of theoretical understanding in designing provably efficient model-based RL algorithms with (non-linear) value function approximation.
\vspace{-3.5mm}
\section{Model Selection in Functional Multi-armed Bandits}
\vspace{-2mm}
\label{sec:bandit}
Consider the problem of sequentially maximizing an unknown function $f^*:\cX\ra \Real$ over a compact domain $\cX \subset \Real^d$. For example, in a machine learning application, $f^*(x)$ can be the validation accuracy of a learning algorithm and $x \in \cX$ is a fixed configuration of (tunable) hyper-parameters of the training algorithm. The objective is to find the hyper-parameter configuration that achieves the highest validation accuracy. An algorithm for this problem chooses, at each round $t$, an input (also called action or arm) $x_t \in \cX$, and subsequently observes a function evaluation (also called reward) $y_t=f^*(x_t) + \epsilon_t$, which is a noisy version of the function value at $x_t$. The action $x_t$ is chosen causally depending upon the history $\lbrace x_1,y_1,\ldots,x_{t-1},y_{t-1}\rbrace$ of arms and reward sequences available before round $t$. 
\begin{assumption}[Sub-Gaussian noise]\label{ass:noise}
The noise sequence $\lbrace\epsilon_t\rbrace_{t \ge 1}$ is conditionally zero-mean, i.e., $\expect{\epsilon_t|\cF_{t-1}}=0$ and $\sigma$-sub-Gaussian for known $\sigma$ ,i.e.,
\beqn
\vspace{-2mm}
\forall t \ge 1,\;\forall \lambda \in \Real, \; \expect{\exp(\lambda \epsilon_t) | \cF_{t-1}} \le \exp\!\left(\!\frac{\lambda^2\sigma^2}{2}\!\right)
\vspace{-1mm}
\eeqn
almost surely, where $\cF_{t-1}:=\sigma(x_1,y_1,\ldots,x_{t-1},y_{t-1},x_t)$ is the $\sigma$-field summarizing the information available just before $y_t$ is observed.
\end{assumption}
This is a mild assumption on the noise (it holds, for instance, for distributions bounded in $[-\sigma, \sigma]$) and is standard in the literature \cite{abbasi2011improved,srinivas2012information,russo2013eluder}.

\textbf{Regret:} The learner's goal is to maximize its (expected) cumulative reward $\sum_{t=1}^{t} f^*(x_t)$ over a time horizon $T$ (not necessarily known a priori) or, equivalently, minimize its cumulative {\em regret} 
\begin{align*}
    \cR_T := \sum\nolimits_{t=1}^{T} \left(f^*(x^*)-f^*(x_t)\right),
\end{align*}
where $x^* \in \argmax_{x\in \cX}f(x)$ is a maximizer of $f$ (assuming the maximum is attained; not necessarily unique). A sublinear growth of $\cR_T$ implies the time-average regret $\cR_T/T \ra 0$ as $T\ra \infty$, implying the
algorithm eventually chooses actions that attain function values close to the optimum most of the time.
\vspace{-2mm}
\subsection{Model Selection Objective}
\vspace{-1mm}
 In the literature, it is assumed that $f^*$ belongs to a known class of functions $\cF$. 
 %  For example, \cite{russo2013eluder} consider
 In this work, in contrast to the standard setting, we do not assume the knowledge of $\cF$. Instead, we are given $M$ nested function classes $\cF_1 \subset \cF_2 \subset \ldots \subset \cF_M$. Among the nested classes $\mathcal{F}_1,..\mathcal{F}_M$, the ones containing $f^*$ is denoted as \emph{realizable} classes, and the ones not containing $f^*$ are dubbed as \emph{non-realizable} classes. The smallest such family where the unknown function $f^*$ lies
is denoted by $\cF_{m^*}$, where $m^* \in [M]$. However, we do not know the index $m^*$, and our goal is to propose adaptive
algorithms such that the regret depends on the complexity of the function class $\cF_{m^*}$. 
In order to achieve this, we need a separability condition on the nested models. 

% \begin{assumption}[Separability]
% \label{ass:sep}
% There exists a constant $\Delta > 0$ such that
% % for any function $f \in \cF_{m^*-1}$ and any input $x \in \cX$, we have
% \beq
% \inf_{f \in\cF_{m^*-1}}\inf_{x \in \cX}\abs{f(x) - f^*(x^*)} \ge \Delta.
% \eeq
% \end{assumption}

\begin{assumption}[Local Separability]
\label{ass:sep}
There exist $\Delta > 0$ and $\eta > 0$  such that
% for any function $f \in \cF_{m^*-1}$ and any input $x \in \cX$, we have
\vspace{-3mm}
\begin{align*}
    \inf_{f \in\cF_{m^*-1}}\inf_{ x_1 \neq x_2: D*(x_1,x_2) \leq \eta}\abs{f(x_1) - f^*(x_2)} \ge \Delta,
\end{align*}
\vspace{-1.5mm}
where\footnote{Here the roles of $x_1$ and $x_2$ are interchangeable without loss of generality.}, $D^*(x_1,x_2) = |f^*(x_1) - f^*(x_2)|$.
\end{assumption}
The above assumption\footnote{We assume that the action set $\mathcal{X}$ is compact and continuous, and so such action pairs $(x_1,x_2)$ always exist, i.e., given any $x_1 \in \mathcal{X}$, an action $x_2$ such that $D^*(x_1,x_2) \leq \eta$ always exists.} ensures that for action pairs $(x_1,x_2)$, where the obtained (expected) rewards are close (since it is generated by $f^*$), there is a gap between the true function $f^*$ and the ones belonging to the function classes not containing $f^*$ (i.e., the non-realizable function classes). Note that we do not require this separability to hold for all actions -- just the ones which are indistinguishable from observing the rewards. Note that separability is needed for model selection since we neither assume any structural assumption on $f^*$, nor on the set $\mathcal{X}$.

We emphasize that separability is quite standard and assumptions of similar nature appear in a wide range of model selection problems, specially in the setting of contextual bandits \cite{ghosh2021modelgen,krishnamurthy2021optimal}. It is also quite standard in statistics, specifically in the area
of clustering and latent variable modelling \cite{balakrishnan2017statistical,mixture-many,ghosh2019max}.

\paragraph{Separability for Lipschitz $f^*$:} If the true function $f^*$ is $1$-Lipschitz. In that setting, the separability assumption takes the following form: for $\Delta >0$ and $\eta>0$,
\begin{align*}
\vspace{-2mm}
    \inf_{f \in\cF_{m^*-1}}\inf_{ x_1 \neq x_2: \|x_1 -x_2\|\leq \eta}\abs{f(x_1) - f^*(x_2)} \ge \Delta
    \vspace{-2mm}
\end{align*}
However, note that the above assumption is quite strong -- any (random) arbitrary algorithm can perform model selection (with the knowledge of $\eta$ and $\Delta$)\footnote{This can be found using standard trick like doubling.} in the following way: first choose action $x_1$. Using $\|x_1 - x_2\| \leq \eta$,  choose $x_2$. Pick any function $f$ belonging to some class $\mathcal{F}_m$ in the nested family and evaluate $\abs{f(x_1) - y_t(x_2)}$, which is a good proxy for $\abs{f(x_1) - f^*(x_2)}$. The algorithm continues to pick different $f \in \mathcal{F}_m$. With the knowledge of $\Delta$, depending on how big $\mathcal{F}_m$ is, the algorithm would be able to identify whether $\mathcal{F}_m$ is realizable or not. Continuing it for all hypothesis classes, it would identify the correct class $\mathcal{F}_{m^*}$. Hence, for structured $f^*$, the problem of model selection with separation is not interesting and we do not consider that setup in this paper.

\paragraph{Separability for Linear $f^*$:} If $f^*$ is linear, the separability assumption is not necessary for model selection. In this setting, $f^*$ is parameterized by some properties of the parameter, such as sparsity and norm, denotes the nested function classes. \cite{ghosh2021problem,foster_model_selection} addresses the linear bandit model selection problem without the separability assumption.

\vspace{-1.5mm}
\subsection{Algorithm: Adaptive Bandit Learning (\texttt{ABL})}
\label{sec:gen}

In this section, we provide a novel model selection algorithm (Algorithm~\ref{algo:generic}) that, over multiple epochs, successively refine the estimate of the true model class $\cF_{m^*}$ where the unknown function $f^*$ lies. At each epoch, we run a fresh instance of a base bandit algorithm for the estimated function class, which we call Bandit Learning. Note that our model selection algorithm works with any provable bandit learning algorithm, and is agnostic to the particular choice of such base algorithm. In what follows, we present a generic description of the base algorithm and then specialize to a special case.

\paragraph{The Base Algorithm} Bandit Learning (Algorithm \ref{algo:base}), in its general form, takes a function class $\cF$ and a confidence level $\delta \in (0,1]$ as its inputs. At each time $t$, it maintains a (high-probability) confidence set $\cC_{t}(\cF,\delta)$ for the unknown function $f^*$, and chooses the most optimistic action with respect to this confidence set, 
\begin{equation}\label{eq:gen-play}
    x_t \in \argmax_{x \in \cX}\max_{f \in \cC_{t}(\cF,\delta)} f(x)\;.
\end{equation}
The confidence set $\cC_{t}(\cF,\delta)$ is constructed using all the data $\lbrace x_s,y_s\rbrace_{s < t}$ gathered in the past. First, a regularized least square estimate of $f^*$ is computed as
$\widehat f_{t} \in \argmin_{f \in \cF} \cL_{t-1}(f), $ where $\cL_{t}(f):=\sum\nolimits_{s=1}^{t} \left(y_s -f(x_s)\right)^2$ is the cumulative squared prediction error.
The confidence set $\cC_{t}(\cF,\delta)$ is then defined as the set of all functions $f \in \cF$ satisfying
\begin{align}\label{eq:gen-conf}
      \sum_{s=1}^{t-1} \left(f(x_s)-\hat f_t(x_s)\right)^2 \le \beta_t(\cF,\delta)\;,
\end{align}
where $\beta_t(\cF,\delta)$ is an appropriately chosen confidence parameter. We now specialize to the bandit learning algorithm of \cite{russo2013eluder} by setting the confidence parameter
\begin{align*}
    \beta_t(\cF,\delta):=8\sigma^2\log\left(2\cN\left(\cF,1/T,\norm{\cdot}_{\infty}\right)/\delta\right) +  2 \left(8+\sqrt{8\sigma^2 \log \left(8t(t+1)/\delta \right)} \right),
\end{align*}
where $\cN(\cF,\alpha,\norm{\cdot}_{\infty})$ is the $(\alpha,\norm{\cdot}_{\infty})$-covering number\footnote{For any $\alpha > 0$, we call $\cF^\alpha$ an $(\alpha,\norm{\cdot}_{\infty})$ cover of the function class $\cF$ if for any $f \in \cF$ there exists an $f'$ in $\cF^\alpha$ such that $\norm{f' - f}_{\infty}:=\sup_{x \in \cX}|f'(x)-f(x)|\le \alpha$.} of $\cF$, one can ensure that $f^*$ lies in the confidence set $\cC_t(\cF,\delta)$ at all time instant $t \ge 1$ with probability at least $1-\delta$. The theoretical guarantees presented in the paper are also for this particular choice of base algorithm.

\begin{algorithm}[t!]
  \caption{Bandit Learning }
  \begin{algorithmic}[1]
 \STATE  \textbf{Input:} Function class $\cF$, confidence level $\delta \in (0,1]$
 %\STATE \textbf{Parameter:} $\lambda > 0$
 \FOR{time $t=1,2,3,\ldots$}
  \STATE Compute an estimate $\widehat f_{t}$ of $f^*$ 
  %\STATE Set confidence parameter $\beta_{t,\delta}^{(i)}\equiv \beta_t\left(\cF_{m^{(i)}},\delta^{(i)}\right)$
  \STATE Construct confidence set $\cC_{t}(\cF,\delta)$ using \eqref{eq:gen-conf}
  \STATE Choose an action $x_t$ using \eqref{eq:gen-play}
  \STATE Observe reward $y_t$
 \ENDFOR
  \end{algorithmic}
  \label{algo:base}
\end{algorithm}

\paragraph{Our Approach--Adaptive Bandit Learning (\texttt{ABL}):}
The description of our model selection algorithm is given in Algorithm~\ref{algo:generic}. We consider doubling epochs -- at each epoch $i \ge 1$, the base algorithm is run over time horizon $t_i=2^i$. At the beginning of $i$-th epoch, using all the data of the
previous epochs, we employ a model selection module as follows. First, we compute, for each class $\cF_m$, the minimum average squared prediction error (via an offline regression oracle)
\begin{align}\label{eq:gen-stat}
    T^{(i)}_m = \min_{f \in \cF_m}\frac{1}{\tau_{i-1}}\sum_{s=1}^{\tau_{i-1}} \left(y_s-f(x_s) \right)^2\;,
\end{align}
where $\tau_{i-1}:=\sum_{j=1}^{i-1}t_j$ denotes the total time elapsed before epoch $i$. 
% Next, we compute the average empirical loss $T_m^{(i)}:=\frac{1}{\tau_{i-1}H}\cL_{\tau_{i-1}}(\widehat P_m^{(i)})$ for the model $\widehat P_m^{(i)}$.
Finally, we compare $T_m^{(i)}$ to a pre-calculated threshold $\gamma$, and pick the function class for which $T_m^{(i)}$ falls
below such threshold (with smallest $m$, see Algorithm~\ref{algo:generic}). After selecting the function class, we run the base algorithm for this class with confidence level $\delta_i=\delta/2^i$. We call the complete procedure Adaptive Bandit Learning (\texttt{ABL}).

\begin{algorithm}[t!]
  \caption{Adaptive Bandit Learning (\texttt{ABL})}
  \begin{algorithmic}[1]
 \STATE  \textbf{Input:} Nested function classes $\cF_1 \subset \cF_2 \subset \ldots \subset \cF_M$, confidence level $\delta \in (0,1]$, threshold $\gamma_i > 0$
  \FOR{epochs $i=1,2 \ldots$}
  \STATE \underline{\textbf{Model Selection:}}
  \STATE Compute elapsed time $\tau_{i-1}=\sum_{j=1}^{i-1}t_j$
  \FOR{function classes $m=1,2 \ldots,M$}
  %\STATE Compute $\widehat f^{(i)}_m = \argmin_{f \in \cF_m}\sum_{s=1}^{\tau_{i-1}} \left(y_s-f(x_s) \right)^2$
  \STATE Compute the minimum average squared prediction error using \eqref{eq:gen-stat}
   \ENDFOR
  \STATE Choose index $m^{(i)}=\min\lbrace m \in [M]: T_m^{(i)} \le \gamma_i\rbrace$ 
  \STATE \underline{\textbf{Model Learning:}}
  \STATE Set epoch length
  $t_i=2^i$, confidence level $\delta_i=\delta/2^i$ 
%   \FOR{time $t=\tau_{i-1}+1,\tau_{i-1}+2 \ldots, \tau_{i-1}+k_i$}
%   \STATE Compute current estimate $\widehat f_{t-1}^{(i)}=\argmin_{f \in \cF_{m^{(i)}}}\sum_{s=\tau_{i-1}+1}^{t-1}\left(y_s-f(x_s)\right)^2$
  %\STATE Set confidence parameter $\beta_{t,\delta}^{(i)}\equiv \beta_t\left(\cF_{m^{(i)}},\delta^{(i)}\right)$
%   \STATE Construct confidence set $\cC_{t-1,\delta,\lambda}^{(i)}=\lbrace f \in \cF_{m^{(i)}}: \sum_{s=\tau_{i-1}+1}^{t-1}\left(f(x_s)-\widehat f_{t-1}^{(i)}(x_s)\right)^2 \leq \beta_t\left(\cF_{m^{(i)}},\delta^{(i)},\lambda\right)\rbrace$
%   \STATE Choose action $x_t = \argmax_{x \in \cX}\max_{f \in \cC_{t-1,\delta}^{(i)}} f(x)$
%   \STATE Observe reward $y_t=f\big(x_t\big)+\varepsilon_t$
%   \ENDFOR
   \STATE Run Bandit Learning (Algorithm \ref{algo:base}) over a time horizon $t_i$ with function class $\cF_{m^{(i)}}$ and confidence level $\delta_i$ as its inputs
    \ENDFOR
  \end{algorithmic}
  \label{algo:generic}
\end{algorithm}
\vspace{-2mm}
\subsection{Performance Guarantee of \texttt{ABL}}
We now provide model selection and regret guarantees of \texttt{ABL} (Algorithm \ref{algo:generic}), when the base algorithm is chosen as \cite{russo2013eluder}. Though the results to be presented in this section are quite general, they do not apply to any arbitrary function classes. In what follows, we will make the following boundedness assumption.
\begin{assumption}[Bounded functions]
\label{ass:bounded}
We assume that $f(x)\in [0,1]$ $\forall$ $x \in \cX$ and $f \in \cF_{m}$ ($\forall$ $m \in [M]$).\footnote{We can extent the range to $[0,c]$ without loss of generality.}
\end{assumption}
It is worth noting that this same assumption is also required in the standard setting, i.e., when the true model class is known ($\cF_{m^*}=\cF$).

We denote by $\log \cN(\cF_m)=\log\left(\cN(\cF_{m},1/T,\norm{\cdot}_{\infty})\right)$ the metric entropy (with scale $1/T$) of the class $\cF_m$. We have the following guarantee for \texttt{ABL}.

\begin{lemma}[Model selection of \texttt{ABL}]
\label{lem:gen_infinite}
Fix a $\delta \in (0,1]$ and $\lambda > 0$. Suppose, Assumptions~\ref{ass:noise}, \ref{ass:sep} and \ref{ass:bounded} hold and we set the threshold $\gamma_i= T_M^{(i)} + C_1$, for a sufficiently small constant $C_1$. Then, with probability at least $1-O(M\delta)$, \texttt{ABL} identifies the correct model class $\mathcal{F}_{m^*}$ from epoch $i \ge i^*$ when the time elapsed before epoch $i^*$ satisfies
\begin{align*}
   \tau_{i^* -1} \ge C \sigma^4 (\log T) \max \left \lbrace \frac{ \log(1/\delta)}{(\frac{\Delta^2}{2}-4\eta)^2} , \log\left(\frac{\cN(\mathcal{F}_{M})}{\delta}\right) \right \rbrace, \end{align*}
provided $\Delta \geq 2\sqrt{2\eta}$, where $C>1$ is a sufficiently large universal constant.
\end{lemma}

\begin{remark}[Dependence on the biggest class]
\label{rem:biggest-class-dependence}
Note that we choose a threshold that depends on the epoch number and the test statistic of the biggest class. Here we crucially exploit the fact that the biggest class always contains the true model class and use this to design the threshold.
\end{remark}

We characterize the complexity of each function class $\cF_m$ by its \emph{eluder dimension}, first
introduced by \cite{russo2013eluder} in the standard setting.

% and since then it has been widely used \cite{osband2014model,ayoub2020model,wang2020provably}.
\begin{definition}[Eluder dimension]\label{def:eluder}
The $\varepsilon$-eluder dimension $\dim_{\cE}(\cF_m,\varepsilon)$ of a function class $\cF$ is the length of the longest sequence $\lbrace x_i\rbrace_{i=1}^{n} \subseteq \cX$ of input points such that for some $\varepsilon' \ge \varepsilon$ and for each $ i \in\lbrace 2,\ldots,n\rbrace$,
\begin{align*}
    \sup_{f_1,f_2 \in \cF}\!\left\lbrace\! (f_1-f_2)(x_i) \;\big | \;  \sqrt{\sum_{j=1}^{i-1}(f_1-f_2)^2(x_i)} \le \varepsilon'\!\right\rbrace > \varepsilon'~.
\end{align*}
\end{definition}
Define $\cF^*=\cF_{m^*}$. Denote by $d_{\cE}(\cF^*)=\dim_{\cE}\left(\cF^*,1/T\right)$, the $(1/T)$-eluder dimension of the (realizable) function class $\cF^*$, where $T$ is the time horizon.
Then, armed with Lemma~\ref{lem:gen_infinite}, we obtain the following regret bound for \texttt{ABL}.
\begin{theorem}[Cumulative regret of \texttt{ABL}]
\label{thm:general}
Suppose the condition of Lemma~\ref{lem:gen_infinite} holds. Then, for any $\delta \in (0,1]$, the regret of \texttt{ABL} for horizon $T$ is
 \begin{align*}
     \cR_T &\le \mathcal{O} \left( \sigma^4 (\log T) \max \left \lbrace \frac{ \log(1/\delta)}{(\frac{\Delta^2}{2}-4\eta)^2} , \log\left(\frac{\cN(\mathcal{F}_{M})}{\delta}\right) \right \rbrace \right) \\
     & + \mathcal{O}\left( d_{\cE}(\cF^*) \log T 
      + c \sqrt{Td_{\cE}(\cF^*)\log(\cN(\cF^*)/\delta) \log^2(T/\delta)}   \right),
 \end{align*}
 with probability at least\footnote{One can choose $\delta = 1/\text{poly}(M)$ to obtain a high-probability bound which only adds an extra $\log M$ factor.} $1- O (M\delta)$.
\end{theorem}

\begin{remark}[Cost of model selection]
We retain the regret bound of \cite{russo2013eluder} in the standard setting, and
the first term in the regret bound captures the cost of model selection -- the cost suffered before accumulating enough samples to infer the correct model class (with high probability). It has weak (logarithmic) dependence on horizon $T$ and hence considered as a minor term, in the setting where $T$ is large. Hence, model selection is essentially \emph{free} upto log factors. Let us now have a close look at this term. It depends on the metric entropy of the biggest model class $\mathcal{F}_M$. This stems from the fact that 
the thresholds $\lbrace\gamma_i\rbrace_{i \ge 1}$ depend on the test statistic of $\cF_M$ (see Remark~\ref{rem:biggest-class-dependence}). We believe that, without additional assumptions, one can't get rid of this (minor) dependence on the complexity of the biggest class.
\end{remark}

The second term is the major one ($\sqrt{T}$ dependence on total number of steps), which essentially is the cost of learning the true kernel $f^*$. Since in this phase, we  basically run the base algorithm for the correct model class, our regret guarantee matches to that of an oracle with the apriori knowledge of the correct class. Note that if we simply run a non model-adaptive algorithm for this problem, the regret would be $\widetilde{\mathcal{O}}(H\sqrt{T d_{\cE}(\cF_M) \log \cN(\cF_M))}$, where $d_{\cE}(\cF_M)$ denotes the eluder dimension of the largest model class $\mathcal{F}_M$. In contrast, by successively testing and thresholding, our algorithm adapts to the complexity of the smallest function class containing the true model class.

\begin{remark}[Requires no knowledge of $(\Delta,\eta)$]
Our algorithm \texttt{ABL} doesn’t require the knowledge of $\Delta$ and $\eta$. Rather, it automatically adapts to these parameters, and the dependence is reflected in the regret expression.  The separation $\Delta$ implies how complex the job of model selection is.  If the separation is small, it is difficult for \texttt{ABL} to separate out the model classes. Hence, it requires additional exploration, and as a result the regret increases. Another 
interesting fact of Theorem~\ref{thm:general} is that it does not require any minimum separation across model classes. This is in sharp contrast with existing results in statistics (see, e.g. \cite{balakrishnan2017statistical,mixture-many}). Even if $\Delta$ is quite small, Theorem~\ref{thm:general} gives a model selection guarantee. Now, the cost of separation appears anyways in the minor term, and hence in the long run, it does not effect the overall performance of the algorithm.

% Bigger $\Delta$ implies easier problem. Note that this can be seen in the model selection cost.

\end{remark}
\vspace{-4mm}
\section{Model Selection in Markov Decision Processes}
\label{sec:rl}
\vspace{-2mm}
An (episodic) MDP is denoted by $\cM(\mathcal{S},\mathcal{A},H, P^*, r)$, where $\mathcal{S}$ is the state space, $\mathcal{A}$ is the action space (both possibly infinite), $H$ is the length of each episode, $P^*:\cS \times \cA \to \Delta(\cS)$ is an (unknown) transition kernel (a function mapping state-action pairs to signed distribution over the state space) and $r: \mathcal{S}\times \mathcal{A} \to [0,1]$ is a (known) reward function. In episodic MDPs, a (deterministic) policy $\pi$ is given by a collection of $H$ functions $(\pi_1,\ldots,\pi_H)$, where each $\pi_h:\cS \to \cA$ maps a state $s$ to an action $a$. In each episode, an initial state $s_1$ is first picked by the environment (assumed to be fixed and history independent). Then, at each step $h \in [H]$, the agent observes the state $s_h$, picks an action $a_h$ according to $\pi_h$, receives a reward $r(s_h,a_h)$, and then transitions to the next state $s_{h+1}$, which is drawn from the conditional distribution $P^*(\cdot| s_h,a_h)$. The episode ends when the terminal state $s_{H+1}$ is reached. For each state-action pair $(s,a)\in \mathcal{S} \times \mathcal{A}$ and step $h \in [H]$, we define action values $Q^{\pi}_h(s,a)$ and and state values $V_h^{\pi}(s)$ corresponding to a policy $\pi$ as \begin{align*}
    Q^{\pi}_h(s,a)\!=\!r(s,a) \!+\! \mathbb{E}\!\left[\!\sum\nolimits_{h'=h+1}^H \!\!\!\!r(s_{h'},\! \pi_{h'}(s_{h'}\!)\!)\!\!\mid\!\! s_h \!=\! s, a_h \!=\! a\!\right]\!,\quad
    V^{\pi}_h(s)\!=\! Q^{\pi}_h\big(s,\pi_h(s)\big)~,
\end{align*}
where the expectation is with respect to the randomness of the transition distribution $P^*$. It is not hard to see that $Q_h^\pi$ and $V_h^\pi$ satisfy the Bellman equations: 
\begin{align*}
    Q^{\pi}_h(s,a) = r(s,a) + (P^* V^{\pi}_{h+1})(s,a)\,,\;\; \forall h \in [H],\quad \text{with $V_{H+1}^\pi(s)=0$ for all $s \in \cS$.}
\end{align*}

% The optimal value function $V^*_h(\cdot)$ and the optimal action-value function $Q^*_h(\cdot,\cdot)$ are given by $V^*_h(s) = \sup_{\pi}V^{\pi}_h(s)$ and $Q^*_h(s,a) = \sup_{\pi}Q^{\pi}_h(s,a)$, where the supremums are over all (non-stationary) policies.  
% Similarly, the Bellman optimality equation is given by 
% \begin{align*}
%     Q^{*}_h(s,a) = \left(r_h + P_h V^{*}_{h+1}\right)(s,a) 
% \end{align*}

A policy $\pi^*$ is said to be optimal if it maximizes the value for all states $s$ and step $h$ simultaneously, and the corresponding optimal value function is denoted by $V^*_{h}(s)=\sup_{\pi }V^{\pi}_{h}(s)$ for all $h \in [H]$, where the supremum is over all (non-stationary) policies. The agent interacts with the environment for $K$ episodes to learn the unknown transition kernel $P^*$ and thus, in turn, the optimal policy $\pi^*$. At each episode $k \ge 1$, the agent chooses a policy $\pi^k := (\pi^k_1,\ldots,\pi^k_H)$ and a trajectory $(s_h^k,a_h^k,r(s_h^k,a_h^k),s_{h+1}^k)_{h\in [H]}$ is generated. The performance of the learning agent is measured by the cumulative (pseudo) regret accumulated
over $K$ episodes, defined as
\begin{align*}
\vspace{-2mm}
    \cR(T) := \sum\nolimits_{k=1}^K\left[ V_1^{*}(s_1^k)-V_1^{\pi^k}(s_1^k)\right],
\end{align*}
where $T=KH$ is total steps in $K$ episodes.

 In this work, we consider general MDPs without any structural assumption on the unknown transition kernel $P^*$.
 In the standard setting \cite{ayoub2020model}, it is assumed that $P^*$ belongs to a known family of transition models $\cP$. Here, in contrast to the standard setting, we do not have the knowledge of $\cP$. Instead, we are given $M$ nested families of transition kernels $\cP_1 \subset \cP_2 \subset \ldots \subset \cP_M$. The smallest such family where the true transition kernel $P^*$ lies
is denoted by $\cP_{m^*}$, where $m^* \in [M]$. However, we do not know the index $m^*$, and our goal is to propose adaptive
algorithms such that the regret depends on the complexity of the family $\cP_{m^*}$. 
We assume a similar separability condition on these nested model classes. 
% Let $V^*_{P,h}$ denote the optimal value function of an MDP with transition kernel $P$ at step $h$, and for each $m \in [M]$, $\cV_m:=\lbrace V^*_{P,h}:P \in \cP_m ,h \in [H]\rbrace$ denote the set of all optimal value functions under the family $\cP_m$. 

% Let $\cM=\lbrace g :\cS \to [0,H]\rbrace$ denotes the set of all $H$-bounded functions on the state space $\cS$ and $\cM^{\text{const}}=\lbrace g:\cS \to c\,, c \in [0,H]\rbrace$ denotes the set of all constant functions on $\cS$ with values in $[0,H]$.
% Note that if $G$ is a continuous random variable supported on the set $\cM$, then $\mathbb{P}[G=g]=0$ for any $g \in \cM^{\text{const}}$.

\begin{assumption}[Local Separability]\label{ass:sep-rl}
There exist constants $\Delta > 0$ and $\eta > 0$  such that for any function $V:\cS \to \Real$,
% for any function $f \in \cF_{m^*-1}$ and any input $x \in \cX$, we have
\begin{align*}
    \inf_{P \in\cP_{m^*-1}} \,\, \inf_{  D*((s_1,a_1),(s_2,a_2)) \leq \eta}\abs{PV(s_1,a_1) - P^*V(s_2,a_2)} \ge \Delta,
\end{align*}
where $(s_1,a_1)\neq (s_2,a_2)$ and $D^*((s_1,a_1),(s_2,a_2)) = |P^*V(s_1,a_1) - P^*V(s_2,a_2)|$.
\end{assumption}
% \begin{assumption}[Separability]\label{ass:sep}
% There exists a $\Delta > 0$ such that for any transition kernel $P \in \cP_{m^*-1}$, value function $V \in \cV_M$ and state-action pair $(s,a)$, we have $\left((PV)(s,a) - (P^*V)(s,a) \right)^2 \ge \Delta$.
% \end{assumption}

% This assumption ensures that given a state-action pair, expected values under the true model is well-separated from expected values under any model from non-realizable classes. We need this to hold uniformly over all states and actions since we do not assume any additional structure over state-action spaces. Note that separability is standard and assumptions of similar nature appear in a wide range of model selection problems, specially in the setting of contextual bandits \cite{ghosh2021modelgen,krishnamurthy2021optimal}. Separability condition is also quite standard in statistics, specifically in the area
% of clustering and latent variable modelling \cite{balakrishnan2017statistical,mixture-many,ghosh2019max}.

This assumption ensures that expected values under the true model is well-separated from those under models from non-realizable classes for two distinct state-action pairs for which values are close under true model. Once again, we need state and action spaces to be compact and continuous to guarantee such pairs always exist.
Note that the assumption might appear to break down for any constant function $V$. However, we will be invoking this assumption with the value functions computed by the learning algorithm (see~\eqref{eq:bellman-rec}). For reward functions that \emph{vary sufficiently} across states and actions, and transition kernels that admit densities, the chance of getting hit by constant value functions is admissibly low. In case the rewards are constant, every policy would anyway incur zero regret rendering the learning problem trivial. The value functions appear in the separability assumption in the first place since we are interested in minimizing the regret. Instead, if one cares only about learning the true model, then separability of transition kernels under some suitable notion of distance (e.g., the KL-divergence) might suffice. 
Note that in \cite{ghosh2021modelgen,krishnamurthy2021optimal}, the regret is defined in terms of the regression function and hence the separability is assumed on the regression function itself. Model selection without separability is kept as an interesting future work.

\vspace{-3mm}
\subsection{Algorithm: Adaptive Reinforcement Learning (\texttt{ARL})}
\vspace{-2mm}
In this section, we provide a novel model selection algorithm \texttt{ARL} (Algorithm~\ref{algo:generic}) that use successive refinements
over epochs. We use {{\ttfamily UCRL-VTR}} algorithm of \cite{ayoub2020model} as our base algorithm, and add a model
selection module at the beginning of each epoch. In other words, over multiple epochs, we
successively refine our estimates of the proper model class where the true transition kernel $P^*$ lies.

\paragraph{The Base Algorithm:} {{\ttfamily UCRL-VTR}}, in its general form, takes a family of transition models $\cP$ and a confidence level $\delta \in (0,1]$ as its input. At each episode $k$, it maintains a (high-probability) confidence set $\cB_{k-1} \subset \cP$ for the unknown model $P^*$ and use it for optimistic planning. First, it finds the transition kernel
    $P_k = \argmax_{P \in \cB_{k-1}} V^*_{P,1}(s_1^k)$,
   where $V^*_{P,h}$ denote the optimal value function of an MDP with transition kernel $P$ at step $h$.
\texttt{UCRL-VTR} then computes, at each step $h$, the optimal value function $V_h^k:=V^*_{P_k,h}$ under the kernel $P_k$ using dynamic programming. Specifically, starting with $V_{H+1}^k(s,a)=0$ for all pairs $(s,a)$, it defines for all steps $h=H$ down to $1$, 
\begin{equation} \label{eq:bellman-rec}
\begin{split}
 Q_h^k(s,a)  = r(s,a) + (P_kV_{h+1}^k)(s,a),\quad
        V_h^k(s) = \max\nolimits_{a \in \cA} Q_h^k(s,a).
        \end{split}
\end{equation}
Then, at each step $h$, {{\ttfamily UCRL-VTR}} takes the action that maximizes the $Q$-function estimate, i,e. it chooses $a_h^k = \argmax_{a \in \cA} Q_h^k(s_h^k,a)$. 
% \begin{align}
%     a_h^k = \argmax_{a \in \cA} Q_h^k(s_h^k,a).
% \end{align}
% If $(\pi_h^k)_{h \in [H]}$ denotes the optimal policy for the model $P_k$, it holds that $a_h^k = \pi_h^k(s_h^k)$.
Now, the confidence set is updated using all the data gathered in the episode. First, {{\ttfamily UCRL-VTR}} computes an estimate of $P^*$ by employing a non-linear value-targeted regression model with data $\big(s_h^j,a_h^j,V_{h+1}^j(s_{h+1}^j)\big)_{j \in [k],h \in [H]}$. 
Note that
% $V_{h+1}^k(s_{h+1}^k \in [0,H]$ almost surely and 
$\E[V_{h+1}^k(s_{h+1}^k)|\cG_{h-1}^k] = (P^* V_{h+1}^k)(s_h^k,a_h^k)$,
where $\cG_{h-1}^k$ denotes the $\sigma$-field summarizing the information available just before $s_{h+1}^k$ is observed. This naturally leads to the estimate $\widehat P_{k} = \argmin_{P \in \cP}\cL_k(P)$, where
\begin{align}\label{eq:gen-est-rl}
     \!\!\cL_k(P)\!:=\!\!\sum\nolimits_{j=1}^{k}\!\sum\nolimits_{h=1}^{H}\! \left(\!V_{h+1}^j(s_{h+1}^j)\!-\!(PV_{h+1}^j)(s_h^j,a_h^j) \!\right)^2.
\end{align}
The confidence set $\cB_k$ is then updated by enumerating the set of all transition kernels $P \in \cP$ satisfying $\sum_{j=1}^{k}\sum_{h=1}^{H} \!\left(\!(PV_{h+1}^j)(s_h^j,a_h^j)\!-\!(\widehat P_{k}V_{h+1}^j)(s_h^j,a_h^j) \!\right)^2 \!\!\le\! \beta_{k}(\delta)$ with the confidence width being defined as $\beta_k(\delta)\!:=\!8H^2\!\log\!\left(\!\frac{2\cN\!\left(\!\cP,\frac{1}{k H},\norm{\cdot}_{\infty,1}\!\right)\!}{\delta}\!\right)\! +\!  4 H^2 \!\left(\!2\!+\!\!\sqrt{\!2 \log\! \left(\!\frac{4kH(kH+1)}{\delta} \!\right)\!} \!\right)$, where $\cN(\cP,\cdot,\cdot)$ denotes the covering number of the family $\cP$.
% $\cN(\cP,1/kH,\norms{\cdot}_{\infty,1})$ denotes the $(1/kH,\norms{\cdot}_{\infty,1})$-covering number of the family $\cP$.
\footnote{For any $\alpha > 0$, $\cP^\alpha$ is an $(\alpha,\norm{\cdot}_{\infty,1})$ cover of $\cP$ if for any $P \in \cP$ there exists an $P'$ in $\cP^\alpha$ such that $\norm{P' - P}_{\infty,1}:=\sup_{s,a}\int_{\cS}|P'(s'|s,a)-P(s'|s,a)|ds' \le \alpha$.} Then, one can show that $P^*$ lies in the confidence set $\cB_k$ in all episodes $k$ with probability at least $1-\delta$. Here, we consider a slight different expression of $\beta_k(\delta)$ as compared to \cite{ayoub2020model}, but the proof essentially follows the same technique. Please refer to Appendix~\ref{app:gen-rl} for further details.

\paragraph{Our Approach:} We consider doubling epochs - at each epoch $i \ge 1$, \texttt{UCRL-VTR} is run for $k_i=2^i$ episodes. At the beginning of $i$-th epoch, using all the data of
previous epochs, we add a model selection module as follows. First, we compute, for each family $\cP_m$, the transition kernel $\widehat P_m^{(i)}$, that minimizes the empirical loss $\cL_{\tau_{i-1}}(P)$ over all $P \in \cP_m$ (see \eqref{eq:gen-est-rl}), where $\tau_{i-1}:=\sum_{j=1}^{k-1}k_j$ denotes the total number of episodes completed before epoch $i$. Next, we compute the average empirical loss $T_m^{(i)}:=\frac{1}{\tau_{i-1}H}\cL_{\tau_{i-1}}(\widehat P_m^{(i)})$ for the model $\widehat P_m^{(i)}$. Finally, we compare $T_m^{(i)}$ to a pre-calculated threshold $\gamma_i$, and pick the transition family for which $T_m^{(i)}$ falls
below such threshold (with smallest $m$, see Algorithm~\ref{algo:generic-rl}). After selecting the family, we run \texttt{UCRL-VTR} for this family with confidence level $\delta_i=\frac{\delta}{2^i}$, where $\delta \in (0,1]$ is a parameter of the algorithm.

\begin{algorithm}[t!]
  \caption{Adaptive Reinforcement Learning -- \texttt{ARL}}
  \begin{algorithmic}[1]
 \STATE  \textbf{Input:} Parameter $\delta$,  function classes $\cP_1 \subset \cP_2 \subset \ldots \subset \cP_M$, thresholds $\lbrace\gamma_i\rbrace_{i\ge 1}$
  \FOR{epochs $i=1,2 \ldots$}
  \STATE Set $\tau_{i-1}=\sum_{j=1}^{i-1}k_j$
  \FOR{function classes $m=1,2 \ldots,M$}
  \STATE Compute $\widehat P^{(i)}_m = \argmin\nolimits_{P \in \cP_m}\!\sum\nolimits_{k=1}^{\tau_{i-1}}\!\sum\nolimits_{h=1}^{H} \!\left(V_{h+1}^k(s_{h+1}^k)\!-\!(PV_{h+1}^k)(s_h^k,a_h^k) \right)^2$
  \STATE Compute $T^{(i)}_m = \frac{1}{\tau_{i-1}H}\!\sum\nolimits_{k=1}^{\tau_{i-1}}\!\sum\nolimits_{h=1}^{H} \!\!\left(V_{h+1}^k(s_{h+1}^k)\!-\!(\widehat P^{(i)}_m V_{h+1}^k)(s_h^k,a_h^k) \right)^2$
   \ENDFOR
  \STATE Set $m^{(i)}=\min\lbrace m \in [M]: T_m^{(i)} \le \gamma_i\rbrace$, $k_i=2^i$ and $\delta_i=\delta/2^i$ 
  \STATE Run {{\ttfamily UCRL-VTR}} for the family $\cP_{m^{(i)}}$ for $k_i$ episodes with confidence level $\delta_i$
    \ENDFOR
  \end{algorithmic}
  \label{algo:generic-rl}
\end{algorithm}

\vspace{-3mm}
\subsection{Performance Guarantee of \texttt{ARL}}
First, we present our main result which states that the model selection procedure of \texttt{ARL} (Algorithm~\ref{algo:generic-rl}) succeeds with high probability after a certain number of epochs. To this end, we denote by $\log \cN(\cP_m)=\log(\cN(\cP_{m},1/T,\norm{\cdot}_{\infty,1}))$ the metric entropy (with scale $1/T$) of the family $\cP_{m}$. We also use the shorthand notation $\cP^*=\cP_{m^*}$.
% \footnote{We have similar results for finite transition families ($|\cP_m| < \infty$), which are deferred to the appendix.}

% have the following guarantee for the model selection procedure of \texttt{ARL-GEN} (Algorithm~\ref{algo:generic}).

\begin{lemma}[Model selection of \texttt{ARL}]
\label{lem:gen_infinite-rl}
Fix a $\delta \in (0,1]$ and suppose Assumption~\ref{ass:sep} holds. Suppose the thresholds are set as $\gamma_i = T^{(i)}_M + C_2,$ for some sufficiently small constant $C_2$.
% Also, suppose that the separation $\Delta$ satisfies
%  \begin{align*}
%   \Delta \geq c_0 + \alpha \left( 2 H^2 \sqrt{2 \log \left(\frac{2KH(KH+1)}{\delta} \right)} + 4H^2\right).
% \end{align*} 
Then, with probability at least $1-O(M\delta)$, \texttt{ARL} identifies the correct model class $\mathcal{P}_{m^*}$ from epoch $i \geq i^*$, where epoch length of $i^*$ satisfies
\begin{align*}
    2^{i^*} \geq C'\log K \max \left \lbrace \frac{H^3 }{(\frac{1}{2}\Delta^2-2H\eta)^2} \log(2/\delta), 4H \log\left(\frac{\cN(\cP_M)}{\delta}\right) \right \rbrace,
\end{align*}
provided $\Delta \geq 2\sqrt{H\eta}$,
for a sufficiently large universal constant $C'>1$.
\end{lemma}

% \textit{Proof idea.} In order to do model selection, we first obtain upper bounds on the test statistics $T^{(i)}_m$ for model classes that includes $\mathcal{P}^*$. We accomplish this by carefully defining a martingale difference sequence that depends on the value function estimates $V^k_{h+1}$, and invoking Azuma-Hoeffding inequality. We then obtain a lower bound on $T^{(i)}_m$ for model classes not containing $\mathcal{P}^*$ by leveraging Assumption~\ref{ass:sep-rl} (separability). Combining the above two bounds yields the desired result. 

\paragraph{Regret Bound:} In order to present our regret bound, we define, for each model model class $\cP_m$, a collection of functions $\cM_m := \left\lbrace f:\cS \times \cA \times \cV_m \to \Real \right\rbrace$ such that any $f \in \cM_m$ satisfies $f(s,a,V) = (PV) (s,a)$ for some $P \in \cP_m$ and $V \in \cV_m$,
% \begin{align*}
%     \exists\; P \in \cP_m \;\text{s.t. for any}\;(s,a,V) \; f(s,a,V) = (PV) (s,a) ,
% \end{align*}
where $\cV_m:=\lbrace V^*_{P,h}:P \in \cP_m ,h \in [H]\rbrace$ denotes the set of optimal value functions under the transition family $\cP_m$.
By one-to-one correspondence, we have $\cM_1 \subset \cM_2 \subset \ldots \subset \cM_M$, and
the complexities of these function classes determine the learning complexity of the RL problem under consideration. We characterize the complexity of each function class $\cM_m$ by its \emph{eluder dimension}, which is defined similarly as Definition \ref{def:eluder}. (We take domain of function class $\cM_m$ to be $\cS\times\cA\times \cV_m$.)

% \begin{definition}[Eluder dimension]\label{def:eluder-rl}
% The $\varepsilon$-eluder dimension $\dim_{\cE}(\cF_m,\varepsilon)$ of the function class $\cF_m$ is the length of the longest sequence $\lbrace(s_i,a_i,V_i)\rbrace_{i=1}^{n} \subseteq \cS \times \cA \times \cV_m$ of state-action-optimal value function tuples under the transition family $\cP_m$ such that for some $\varepsilon' \ge \varepsilon$ and for each $ i \in\lbrace 2,\ldots,n\rbrace$, we have $\sup\nolimits_{f_1,f_2 \in \cF_m} \big(f_1-f_2\big)(s_i,a_i,V_i) > \varepsilon'~,$ given that $\sqrt{\sum\nolimits_{j=1}^{i-1}(f_1-f_2)^2(s_i,a_i,V_i)} \leq \varepsilon'$.
% \end{definition}

% The notion of eluder dimension was first
% introduced by \cite{russo2013eluder} to characterize the complexity of function classes, and since then it has been widely used \cite{osband2014model,ayoub2020model,wang2020provably}.

% We denote by $d_{\cE}^*=\dim_{\cE}(\cF_{m^*},1/T)$ the $(1/T)$-eluder dimension of the function class $\cF_{m^*}$ corresponding to the (realizable) family $\cP_{m^*}$. Furthermore, we denote by $\mathbb{M}^*:=\mathbb{M}_{m^*}$ the metric entropy of the family $\cP_{m^*}$.

We define $\cM^*=\cM_{m^*}$, and denote by $d_{\cE}(\cM^*)=\dim_{\cE}\left(\cM^*,1/T\right)$, the $(1/T)$-eluder dimension of the (realizable) function class $\cM^*$, where $T$ is the time horizon. 
Then, armed with Lemma~\ref{lem:gen_infinite-rl}, we obtain the following regret bound.

% \footnote{For ease
% of representation, we assume the transition kernel $P^*$  to be fixed for all steps $h$. Our results extends naturally to the setting, where there are $H$ different kernels. This would only add a multiplicative $\sqrt{H}$ factor in the regret bound \cite{jin2019provably}. Moreover, our results can be extended to the setting where the rewards are also unknown.}
% and model selection needs to be done for that too.}
\begin{theorem}[Cumulative regret of \texttt{ARL}]
\label{thm:general-rl}
Suppose the conditions of Lemma~\ref{lem:gen_infinite-rl} hold. Then, for any $\delta \!\in\! (0,1]$, running \texttt{ARL} for $K$ episodes yields a regret bound
\begin{align*}
   \cR(T) &= \mathcal{O} \left(\log K \max \left \lbrace \frac{H^4  \log(1/\delta)}{(\frac{\Delta^2}{2}-2H\eta)^2}, H^2 \log\left(\frac{\cN(\cP_M)}{\delta}\right) \right \rbrace \right)\\ &\quad + \cO \left(H^2d_{\cE}(\cM^*) \log K+H\sqrt{Td_{\cE}(\cM^*)\log(\cN(\cP^*)/\delta)\log K \log(T/\delta)} \right)~.
\end{align*}
 with probability at least $1- O(M\delta)$.
\end{theorem}
Similar to Theorem~\ref{thm:general}, the first term in the regret bound captures the cost of model selection, having weak (logarithmic) dependence on the number of episodes $K$ and hence considered as a minor term, in the setting where $K$ is large. Hence, model selection is essentially \emph{free} upto log factors. The second term is the major one ($\sqrt{T}$ dependence on total number of steps), which essentially is the cost of learning the true kernel $P^*$. Since in this phase, we  basically run \texttt{UCRL-VTR} for the correct model class, our regret guarantee matches to that of an oracle with the apriori knowledge of the correct class. \texttt{ARL} doesnot require the knowledge of $(\Delta,\eta)$ and it adapts to the complexity of the problem.
\vspace{-4mm}
\section{Conclusion}
\vspace{-2mm}
We address the problem of model selection for MAB and MDP and propose algorithms that obtains regret similar to an oracle who knows the true model class apriori. Our algorithms leverage the separability conditions crucially, and removing them is kept as a future work. 

\section*{Acknowledgements}
We thank anonymous reviewers for their useful comments. Moreover, we would like to thank Prof. Kannan Ramchandran (EECS, UC Berkeley) for insightful discussions regarding the topic of model selection. SRC is grateful to a CISE postdoctoral fellowship of Boston University.

\bibliographystyle{splncs04}
\bibliography{refs.bib}

\clearpage
\begin{appendix}
\begin{center}
    \textbf{\LARGE Supplementary Material}
\end{center}

\section{Proof for Bandits}
First, we find concentration bounds on the test statistics $T_m^{(i)}$ for all epochs $i \ge 1$ and class indexes $m \in [M]$, which are crucial to prove the model selection guarantee (Lemma \ref{lem:gen_infinite}) of \texttt{ABL}.

\subsection{Realizable model classes}
We find an upper and a lower bound on $T_m^{(i)}$.
\paragraph{1. Upper bound on $T_m^{(i)}$:} 
Fix a class index $m \ge m^*$. In this case, the true model $f^* \in \cF_m$. Let $\hat f_m^{(i)}$ attains the minimum \eqref{eq:gen-stat}. Then, the we can upper bound the empirical risk at epoch $i$ as
\begin{align}\label{eq:realizable}
    T_m^{(i)}= \frac{1}{\tau_{i-1}}\sum_{t=1}^{\tau_{i-1}} \left(y_t- \hat f_m^{(i)}(x_t) \right)^2\le \frac{1}{\tau_{i-1}}\sum_{t=1}^{\tau_{i-1}} \left(y_t- f^*(x_t) \right)^2=\frac{1}{\tau_{i-1}}\sum_{t=1}^{\tau_{i-1}}\epsilon_t^2.
\end{align}
Now, since the noise $\epsilon_t$ is (conditionally) sub-Gaussian, the quantity $\epsilon_t^2$ is (conditionally) sub-exponential, and using sub-exponential concentration bound, we obtain for all epochs $i \ge 1$,
\begin{align}
    \frac{1}{\tau_{i-1}}\sum_{t=1}^{\tau_{i-1}}\epsilon_t^2 &\le \frac{1}{\tau_{i-1}}\sum_{t=1}^{\tau_{i-1}}\mathbb{E}(\epsilon_t^2 | \mathcal{G}_{t-1}) + \mathcal{O}\left( \sigma^2 \sqrt{\frac{\log(2^i/\delta)}{\tau_{i-1}}} \right)\nonumber\\ &= v^{(i)} + \mathcal{O}\left( \sigma^2 \sqrt{\frac{\log(2^i/\delta)}{\tau_{i-1}}} \right)
    \label{eq:realizable_upper}
\end{align}
with probability at least $1-\delta$, where, $v^{(i)} = \frac{1}{\tau_{i-1}}\sum_{t=1}^{\tau_{i-1}}\mathbb{E}(\epsilon_t^2 | \mathcal{G}_{t-1})$ and $\cG_{t-1}$ denotes the $\sigma$-field summarising all the information available just before $y_t$ is observed.
\paragraph{2. Lower bound on $T_m^{(i)}$}
Here, we decompose $T_m^{(i)}$ in the following way. We write $T_m^{(i)} = T^{(i)}_{m,1} + T^{(i)}_{m,2}+ T^{(i)}_{m,3}$, where
\begin{align*}
    T^{(i)}_{m,1}  &= \frac{1}{\tau_{i-1}}\sum_{t=1}^{\tau_{i-1}} \left(y_t-f^*(x_t) \right)^2,\\
    T^{(i)}_{m,2} & = \frac{1}{\tau_{i-1}}\sum_{t=1}^{\tau_{i-1}} \left(f^*(x_t)- \hat f^{(i)}_m (x_t) \right)^2,\\
    T^{(i)}_{m,3} &= \frac{1}{\tau_{i-1}}\sum_{t=1}^{\tau_{i-1}} 2 \left(f^*(x_t)- \hat f^{(i)}_m (x_t) \right)\left(y_t-f^*(x_t) \right).
\end{align*}
To bound the first term, we have for all $i \ge 1$,
\begin{align*}
    T^{(i)}_{m,1} &= \frac{1}{\tau_{i-1}}\sum_{t=1}^{\tau_{i-1}} \left(y_t-f^*(x_t) \right)^2 \\
    &= \frac{1}{\tau_{i-1}}\sum_{t=1}^{\tau_{i-1}} \epsilon_t^2 \\
    & \ge v^{(i)} - \mathcal{O}\left( \sigma^2 \sqrt{\frac{\log(2^i/\delta)}{\tau_{i-1}}} \right)
\end{align*}
with probability at least $1-\delta$, using the same sub-exponential concentration inequality.

For the second term, since it is a squared term, we trivially have
\begin{align*}
    T^{(i)}_{m,2} = \frac{1}{\tau_{i-1}}\sum_{t=1}^{\tau_{i-1}} \left(f^*(x_t)- \hat f^{(i)}_m (x_t) \right)^2 \ge 0.
\end{align*}
Finally, for the third term, we have, with probability at least $1-\delta$,
\begin{align*}
    T^{(i)}_{m,3} = \frac{1}{\tau_{i-1}}\sum_{t=1}^{\tau_{i-1}} 2 \left(f^*(x_t)- \hat f^{(i)}_m (x_t) \right)\left(y_t-f^*(x_t) \right)  = \frac{1}{\tau_{i-1}}\sum_{t=1}^{\tau_{i-1}} 2 \left(f^*(x_t)- \hat f^{(i)}_m (x_t) \right) \epsilon_t.
\end{align*}
Let us fix a function $f \in \mathcal{F}_m$, where $\mathcal{F}_m$ is realizable. Define the random variable $Z_s^f=2(f^*(x_s) - f(x_s))\epsilon_t$. Note that $Z_s^f$
is zero-mean and $2\sigma|f^*(x_s) - f(x_s)|$ sub-Gaussian conditioned on $\mathcal{G}_{s-1}$. Therefore, for any $\lambda < 0$, with probability at least $1-\delta$, we have
\begin{align*}
\forall t \ge 1,\quad  \sum_{s=1}^{t}Z_s^f \ge \frac{1}{\lambda}\log(1/\delta) +\frac{\lambda}{2}\cdot 4\sigma^2 \sum_{s=1}^{t} (f^*(x_s) - f(x_s))^2 . 
\end{align*}
Setting $\lambda = -1/4\sigma^2$, we obtain 
\begin{align}
\forall t \ge 1,\quad \sum_{s=1}^{t} Z_s^f \ge -4\sigma^2\log(1/\delta) -\frac{1}{2}\cdot  \sum_{s=1}^{t} (f^*(x_s) - f(x_s))^2 .
\label{eq:use-bandit-1}
\end{align}
The above calculations hold for a fixed function in the realizable class. Now, we consider all functions in the class $\mathcal{F}_m$.  We consider both the finite and the infinite cases.

\paragraph{Case 1-finite $\mathcal{F}_m$}  Taking a union bound over all $f\in \mathcal{F}_m$, we obtain from \eqref{eq:use-bandit-1}
\begin{align*}
   \forall i \ge 1,\quad T^{(i)}_{m,3} &\ge \frac{-4\sigma^2}{\tau_{i-1} }\log(|\cF_m|/\delta)+ \frac{1}{-2 \tau_{i-1}}\cdot  \sum_{t=1}^{\tau_{i-1}}\left(f^*(x_t)- \hat f_m^{(i)}(x_t) \right)^2 \\
    &\ge \frac{-4\sigma^2}{\tau_{i-1} }\log(|\cF_m|/\delta) - \frac{1}{2} T^{(i)}_{m,2}.
\end{align*}
Now, combining the above three terms, we obtain for all $i \ge 1$,
\begin{align}
    T_m^{(i)} &= T^{(i)}_{m,1} + T^{(i)}_{m,2} + T^{(i)}_{m,3} \nonumber\\
    & \ge v^{(i)} - \mathcal{O}\left( \sigma^2 \sqrt{\frac{\log(2^i/\delta)}{\tau_{i-1}}} \right) + T^{(i)}_{m,2} - \frac{4\sigma^2}{\tau_{i-1} }\log(|\cF_m|/\delta) - \frac{1}{2} T^{(i)}_{m,2}\nonumber \\
    &\ge v^{(i)} - \mathcal{O}\left( \sigma^2 \sqrt{\frac{\log(2^i/\delta)}{\tau_{i-1}}} \right) + \frac{1}{2}T^{(i)}_{m,2} - \frac{4\sigma^2}{\tau_{i-1} }\log(|\cF_m|/\delta) \nonumber\\
    &\ge v^{(i)} - \mathcal{O}\left( \sigma^2 \sqrt{\frac{\log(2^i/\delta)}{\tau_{i-1}}} \right) - \frac{4\sigma^2}{\tau_{i-1} }\log(|\cF_m|/\delta),
    \label{eq:realizable_lower_finite}
\end{align}
with probability at least $1-c_1\delta$ ($c_1$ is a positive universal constant). Here, the final inequality uses the fact that $T_{m,2}^{(i)} \ge 0$.

%  \begin{align*}
%     T_m^{(i)} \ge v^{(i)} - \mathcal{O}\left( \sigma^2 \sqrt{\frac{\log(2^i|\mathcal{F}_m|/\delta)}{\tau_{i-1}}} \right) - \frac{4\sigma^2}{\tau_{i-1} }\log(2^i|\mathcal{F}_m|/\delta)
%  \end{align*}
% with probability at least $1-c\delta$.
 
% \begin{align}
%     T^{(i)}_{m,3} \ge \frac{1}{\tau_{i-1} \lambda}\log(2^i/\delta)+ \frac{\lambda}{2 \tau_{i-1}}\cdot 4\sigma^2 \sum_{t=1}^{\tau_{i-1}}\left(f^*(x_t)- \hat f (x_t) \right)^2
%     \end{align}
% Setting $\lambda = -1/4\sigma^2$, we obtain

\paragraph{Case 2-infinite $\mathcal{F}_m$}

Fix some $\alpha > 0$. Let $\cF_m^\alpha$ denotes an $(\alpha,\norm{\cdot}_{\infty})$ cover of $\cF_m$, i.e., for any $f \in \cF_m$, there exists an $f^\alpha$ in $\cF_m^\alpha$ such that $\norm{f^\alpha - f}_{\infty} \le \alpha$. Now, we take a union bound over all $f^\alpha \in \cF_m^\alpha$ in \eqref{eq:use-bandit-1} to obtain that
\begin{align*}
    \forall t \ge 1,\;\; \forall f^\alpha \in \cF_m^\alpha,\quad\sum_{s=1}^{t} Z_s^{f^\alpha} \ge -4\sigma^2\log\left(\frac{|\cF_m^\alpha|}{\delta}\right)- \frac{1}{2} \sum_{s=1}^{t} (f^*(x_s) - f^\alpha(x_s))^2 ~.
\end{align*}
with probability at least $1-\delta$, and thus, in turn, 
\begin{align}
    \forall t \ge 1,\;\; \forall f\in \cF_m,\quad\sum_{s=1}^{t} Z_s^{f} \ge -4\sigma^2\log\left(\frac{|\cF_m^\alpha|}{\delta}\right)- \frac{1}{2} \sum_{s=1}^{t} (f^*(x_s) - f(x_s))^2+\zeta_t^{\alpha}(f).
    \label{eq:comb-one-conf-bandit}
\end{align}
with probability at least $1-\delta$,
where $\zeta_k^{\alpha}(f)$ denotes the discretization error:
\begin{align*}
    &\zeta_t^{\alpha}(f)\\ =& \sum_{s=1}^{t}\left(Z_s^{f}-  Z_s^{f^\alpha}\right) + \frac{1}{2}(f^*(x_s) - f(x_s))^2 - \frac{1}{2}(f^*(x_s) - f^\alpha(x_s))^2\\
    &=\sum_{s=1}^{t}\left( 2\epsilon_s(f^\alpha(x_s)-f(x_s))+\frac{1}{2}f(x_s)^2-\frac{1}{2}f^\alpha(x_s)^2+f^*(x_s)(f^\alpha(x_s)-f(x_s))\right).
\end{align*}
Since $\norm{f-f^\alpha}_{\infty} \le 1$, it is enough to take $\alpha \leq 1$. We have
\begin{align*}
    \left|f^\alpha(x) ^2- f(x)^2 \right| &\le \max_{|\xi| \le \alpha } \left | \left(f(x)+\xi\right)^2-f(x)^2 \right|\\ &\le 2 \alpha  +\alpha^2~.
\end{align*}
Therefore, we can upper bound the discretization error as
\begin{align*}
    |\zeta_t^{\alpha}(f)| &\le 2 \alpha  \sum_{s=1}^{t}|\epsilon_s| + \sum_{s=1}^{t} \left(2\alpha  + \frac{\alpha^2}{2}\right).
\end{align*}
Since $\epsilon_t$ is $\sigma$-sub-Gaussian conditioned on $\cG_{t-1}$, we have 
\begin{align*}
    \forall t \ge 1,,\quad |\epsilon_t| \le \sigma\sqrt{2 \log \left(\frac{2t(t+1)}{\delta} \right)}
\end{align*}
with probability at least $1-\delta$.
Therefore, with probability at least $1-\delta$, the discretization error is bounded for all $t \ge 1$ as
\begin{align*}
    |\zeta_t^{\alpha}(f)| &\le t \left( 2\alpha \sigma \sqrt{2 \log \left(\frac{2t(t+1)}{\delta} \right)} + 2\alpha +\frac{\alpha^2}{2}\right)\\
    & \le \alpha t \left( 2 \sigma \sqrt{2 \log \left(\frac{2t(t+1)}{\delta} \right)} + 3\right),
\end{align*}
where the last step holds for any $\alpha \leq 1$. Therefore, from \eqref{eq:comb-one-conf-bandit}, we have
\begin{align}
   \forall t \ge 1,\;\; \forall f\in \cF_m,\quad\sum_{s=1}^{t} Z_s^{f} &\ge -4\sigma^2\log\left(\frac{|\cF_m^\alpha|}{\delta}\right)- \frac{1}{2} (f^*(x_s)-f(x_s))^2\nonumber\\ &  \quad \quad- \alpha t \left( 2 \sigma \sqrt{2 \log \left(\frac{2t(t+1)}{\delta} \right)} + 3\right),
   \label{eq:infinite-bandit}
\end{align}
with probability at least $1-2\delta$.

A similar argument as in \eqref{eq:realizable_lower_finite} yields
\begin{align}
\forall i \ge 1,\quad T_m^{(i)}&\ge v^{(i)} - \mathcal{O}\left( \sigma^2 \sqrt{\frac{\log(2^i/\delta)}{\tau_{i-1}}} \right) - \frac{4\sigma^2}{\tau_{i-1} }\log(\cN(\cF_m,\alpha,\norm{\cdot}_\infty)/\delta)\nonumber\\
&\quad\quad\quad - \alpha  \left( 2 \sigma \sqrt{2 \log \left(\frac{2\tau_{i-1}(\tau_{i-1}+1)}{\delta} \right)} + 3\right),
    \label{eq:realizable_lower_infinite}
\end{align}
with probability at least $1-c_2\delta$ ($c_2$ is a positive universal constant).

\subsection{Non realizable model classes}
Fix a class index $m < m^*$. In this case, the true model $f^* \notin \cF_m$. We can decompose the empirical risk $T_m^{(i)}$ at epoch $i$ as the sum of the following three terms: 
\begin{align*}
    S^{(i)}_{m,1}  &= \frac{1}{\tau_{i-1}}\sum_{t=1}^{\tau_{i-1}} \left(y_t-f^*(\x) \right)^2,\\
    S^{(i)}_{m,2} & = \frac{1}{\tau_{i-1}}\sum_{t=1}^{\tau_{i-1}} \left(f^*(\x)- \hat f^{(i)}_m (x_t) \right)^2,\\
    S^{(i)}_{m,3} &= \frac{1}{\tau_{i-1}}\sum_{t=1}^{\tau_{i-1}} 2 \left(f^*(\x)- \hat f^{(i)}_m (x_t) \right)\left(y_t-f^*(\x) \right).
\end{align*}
where $\x$ satisfies
\begin{align*}
    D^*(x_t,\x) = |f^*(x_t) - f^*(\x)| \le \eta
\end{align*}
To bound the first term, note that
\begin{align*}
   \left(y_t-f^*(\x) \right)^2 &=  \left(y_t- f^*(x_t) \right)^2+\left(f^*(x_t)-f^*(\x) \right)^2 +2\left(y_t-  f^*(x_t) \right)\left(f^*(x_t)-f^*(\x) \right)\\ & \ge \left(y_t- f^*(x_t) \right)^2+2\left(f^*(x_t)-f^*(\x) \right)\epsilon_t\\
   &= \epsilon_t^2+2\left(f^*(x_t)-f^*(\x) \right)\epsilon_t~.
\end{align*}
Define $Y_t=2\left(f^*(x_t)-f^*(\x) \right)\epsilon_t$. Let $\cG_{t-1}$ denotes the $\sigma$-field summarising all the information available just before $y_t$ is observed. Note that $Y_t$ is $\cG_{t}$-measurable and $\E[Y_t|\cG_{t-1}]=0$. Moreover, since $f^*(x) \in [0,1]$ and $\epsilon_t$ is $\sigma$-sub-Gaussian, $Y_t$ is $2\sigma$ sub-Gaussian conditioned on $\cG_{t-1}$. Hence, by Azuma-Hoeffding inequality, with probability at least $1-\delta$,
\begin{align*}
    \sum_{t=1}^{\tau_{i-1}}Y_t \ge -2\sigma \sqrt{2\tau_{i-1}\log(1/\delta)}.
\end{align*}
Now, similar to \eqref{eq:realizable}, with probability at least $1-\delta/2^i$, we obtain
\begin{align*}
    \forall i \ge 1, \quad S^{(i)}_{m,1} \ge v^{(i)} - \mathcal{O}\left( \sigma^2 \sqrt{\frac{\log(2^i/\delta)}{\tau_{i-1}}} \right) - 2\sigma \sqrt{\frac{2\log(2^{i}/\delta)}{\tau_{i-1}}}.
\end{align*}

The second term is bounded by Assumption \ref{ass:sep} as \begin{align*}
\forall i \ge 1, \quad    S^{(i)}_{m,2} \ge \Delta^2~.
\end{align*}
Now, we turn to bound the term $S^{(i)}_{m,3}$. For any fixed $f \in \cF_m$, define the following
\begin{align*}
    Z_t^f:=2 \left(f^*(\x)-  f(x_t) \right)\left(y_t-f^*(\x) \right)=2 \left(f^*(\x)-  f(x_t) \right)\left(f^*(x_t)-f^*(\x)+\epsilon_t \right)~.
\end{align*}
Note that $Z_t^f$ is $\cG_{t}$-measurable, and 
\begin{align*}
   \E[Z_t^f|\cG_{t-1}]=2\left(f^*(\x)- f (x_t) \right)\left(f^*(x_t)- f^* (\x) \right).
\end{align*}
Then, we have
\begin{align}
    Z_t^f-\E[Z_t^f|\cG_{t-1}]=2\left(f^*(\x)- f (x_t) \right)\epsilon_t~,
\end{align}
which is $2\sigma|f^*(\x)- f (x_t)|$-sub-Gaussian conditioned on $\cG_{t-1}$. Therefore, for any $\lambda < 0$, with probability at least $1-\delta$, we have for all epoch $i \ge 1$,
\begin{align*}
    \sum_{t=1}^{\tau_{i-1}} Z_t^f &\ge \sum_{t=1}^{\tau_{i-1}}-2\left(f^*(\x)- f (x_t) \right)\left(f^*(\x)- f^* (x_t) \right)+\frac{1}{\lambda}\log(1/\delta)+ \frac{\lambda}{2}\cdot 4\sigma^2 \sum_{t=1}^{\tau_{i-1}}\left(f^*(\x)-  f (x_t) \right)^2\\
    &\ge \sum_{t=1}^{\tau_{i-1}}-2|\left(f^*(\x)- f (x_t) \right)| |\left(f^*(\x)- f^* (x_t) \right)|+\frac{1}{\lambda}\log(1/\delta)+ \frac{\lambda}{2}\cdot 4\sigma^2 \sum_{t=1}^{\tau_{i-1}}\left(f^*(\x)-  f (x_t) \right)^2 \\
    &\ge \sum_{t=1}^{\tau_{i-1}}-4|\left(f^*(\x)- f^* (x_t) \right)|+\frac{1}{\lambda}\log(1/\delta)+ \frac{\lambda}{2}\cdot 4\sigma^2 \sum_{\tau=1}^{t}\left(f^*(\x)-  f (x_t) \right)^2 \\
    &\ge \sum_{t=1}^{\tau_{i-1}}-4\eta+\frac{1}{\lambda}\log(1/\delta)+ \frac{\lambda}{2}\cdot 4\sigma^2 \sum_{t=1}^{\tau_{i-1}}\left(f^*(\x)-  f (x_t) \right)^2 \\
    & \ge -4\eta \,\tau_{i-1} +\frac{1}{\lambda}\log(1/\delta)+ \frac{\lambda}{2}\cdot 4\sigma^2 \sum_{t=1}^{\tau_{i-1}}\left(f^*(\x)-  f (x_t) \right)^2.
\end{align*}

Setting $\lambda = -1/4\sigma^2$, we obtain for any fixed $f \in \cF_m$, the following:
\begin{align}
    \forall i \ge 1, \quad\sum_{t=1}^{\tau_{i-1}} Z_t^{f} \ge -4\eta \,\tau_{i-1} -4\sigma^2\log\left(\frac{1}{\delta}\right)- \frac{1}{2} \sum_{t=1}^{\tau_{i-1}}\left(f^*(\x)-  f (x_t) \right)^2~.
    \label{eq:subg-conc}
\end{align}
with probability at least $1-\delta$. We consider both the cases -- when $\cF_m$ is finite and when $\cF_m$ is infinite.

\paragraph{Case 1 -- finite $\cF_m$:}

We take a union bound over all $f \in \cF_m$ in \eqref{eq:subg-conc} to obtain that
\begin{align}
   \forall f \in \cF_m,\quad\quad\sum_{t=1}^{\tau_{i-1}} Z_t^{f} \ge -4\eta \,\tau_{i-1} -4\sigma^2\log\left(\frac{\abs{\cF_m}}{\delta}\right)- \frac{1}{2} \sum_{t=1}^{\tau_{i-1}}\left(f^*(\x)-  f (x_t) \right)^2~.
    \label{eq:finite}
\end{align}
with probability at least $1-\delta$. Note that $\hat f_m^{(i)} \in \cF_m$ for all $i \ge 1$. Then, from \eqref{eq:finite}, we have
\begin{align*}
     \forall i \ge 1,\quad S_{m,3}^{(i)} \ge -4\eta 
    -\frac{4\sigma^2}{\tau_{i-1}} \log\left(\frac{|\cF_m|}{\delta}\right) -\frac{1}{2} S_{m,2}^{(i)}.
\end{align*}
with probability at least $1-\delta$. Now, combining all the three terms together and using a union bound, we obtain the following for any class index $m < m^*$:
\begin{align}
   \forall i \ge 1,\quad T_m^{(i)} &\ge v^{(i)} + (\frac{1}{2}\Delta^2 -4\eta) - 2\sigma \sqrt{\frac{2\log(2^{i}/\delta)}{\tau_{i-1}}} -\mathcal{O}\left( \sigma^2 \sqrt{\frac{\log(2^i/\delta)}{\tau_{i-1}}} \right)\nonumber  \\&\quad\quad-\frac{4\sigma^2}{\tau_{i-1}} \log\left(\frac{|\cF_m|}{\delta}\right).
   \label{eq:finite-non-realizable}
 \end{align}
with probability at least $1-c_3\delta$ ($c_3$ is a positive universal constant).
.

\paragraph{Case 2 -- infinite $\cF_m$:} A similar argument as in \eqref{eq:realizable_lower_infinite} yields
\begin{align}
\forall i \ge 1,\quad T_m^{(i)}&\ge v^{(i)}(\frac{1}{2}\Delta^2 -4\eta) - 2\sigma \sqrt{\frac{2\log(2^{i}/\delta)}{\tau_{i-1}}} - \mathcal{O}\left( \sigma^2 \sqrt{\frac{\log(2^i/\delta)}{\tau_{i-1}}} \right) \nonumber\\
&\quad\quad\quad- \frac{4\sigma^2}{\tau_{i-1} }\log(\cN(\cF_m,\alpha,\norm{\cdot}_\infty)/\delta)\nonumber\\
&\quad\quad\quad - \alpha  \left( 2 \sigma \sqrt{2 \log \left(\frac{2\tau_{i-1}(\tau_{i-1}+1)}{\delta} \right)} + 3\right),
    \label{eq:infinite-non-realizable}
\end{align}
with probability at least $1-c_4\delta$ ($c_4$ is a positive universal constant).

% We follow a similar procedure, using \eqref{eq:infinite}, to obtain the following for any class index $m \le m^*-1$:
% \begin{align}
%      \forall i \ge 1,\quad T_{m}^{(i)} &\ge \sigma^2 + \frac{1}{2}\Delta - H^{3/2}\sqrt{\frac{2\log(2^{i}/\delta)}{\tau_{i-1}}}
%     -\frac{4H}{\tau_{i-1}} \log\left(\frac{\cN(\cP_m,\alpha,\norm{\cdot}_{\infty,1})}{\delta}\right)\nonumber\\ &\quad\quad- \alpha \left( 2 H^2 \sqrt{2 \log \left(\frac{2\tau_{i-1}H(\tau_{i-1}H+1)}{\delta} \right)} + 4H^2\right)
%     \label{eq:infinite-non-realizable}
% \end{align}
% with probability at least $1-3\delta$. 

\subsubsection{Proof of Lemma \ref{lem:gen_infinite}}
We are now ready to prove Lemma \ref{lem:gen_infinite}, which presents the model selection Guarantee of \texttt{ABL} for infinite model classes $\lbrace \cF_m \rbrace_{m \in [M]}$. Here, at the same time, we prove a similar (and simpler) result for finite model classes also. 

First, note that we consider doubling epochs $t_i=2^i$, which implies $\tau_{i-1} = \sum_{j=1}^{i-1}t_j = \sum_{j=1}^{i-1} 2^j = 2^i -1$. With this, the number of epochs is given by $N=\lceil \log_2(T+1) -1 \rceil = \mathcal{O}(\log T)$. Let us now consider finite model classes $\lbrace \cF_m \rbrace_{m \in [M]}$.

\paragraph{Case 1 -- finite model classes:}
Here, we combine equations \eqref{eq:realizable}, \eqref{eq:realizable_lower_finite} and \eqref{eq:finite-non-realizable} to obtain a threshold $\gamma_i$. We take a union bound over all $m \in [M]$ and use that $\log(2^i/\delta) \le N\log(2/\delta)$ for all $i$ and $|\cF_m| \le |\cF_{M}|$ for all $m $.
%  With this, equations~\eqref{eqn:realizable-finite} and \eqref{eqn:nonrealizable-finite} can be rewritten as
% \begin{align*}
%     \forall m \ge m^*,\;\forall i \ge 1,\quad T_m^{(i)} &\le \sigma^2  + H^{3/2}\sqrt{\frac{2 N\log(2/\delta)}{\tau_{i-1}}} \; \text{and}\\
%     \forall m \le m^*-1,\;\forall i \ge 1,\quad T_m^{(i)} &\ge \sigma^2 + \frac{1}{2}\Delta - H^{3/2}\sqrt{\frac{2 N\log(2/\delta)}{\tau_{i-1}}}  -\frac{4H}{\tau_{i-1}} \log\left(\frac{|\cP_M|}{\delta}\right) 
% \end{align*}
Now, suppose for some epoch $i^*$,  satisfies 
\begin{align*}
    2^{i^*} \ge C \log T \max \left \lbrace \frac{\sigma^4}{(\frac{\Delta^2}{2}-4\eta)^2} \log(1/\delta), \sigma^4 \log\left(\frac{|\mathcal{F}_{M}|}{\delta}\right) \right \rbrace.
\end{align*}
where $C$ is a sufficiently large universal constant and $\Delta \ge 2\sqrt{2\eta}$.

Note that the largest class, $\mathcal{F}_M$ is realizable. Then, we have
\begin{align*}
    T_M^{(i)} \le v^{(i)} + \mathcal{O}\left( \sigma^2 \sqrt{\frac{\log(2^i/\delta)}{\tau_{i-1}}} \right).
\end{align*}
Using the above, let us first consider the non-realizable classes. We now have
\begin{align*}
    \forall m < m^*,\;\forall i \ge i^*,\quad T_m^{(i)} &\ge  v^{(i)} + (\frac{1}{2}\Delta^2 -4\eta) - 2\sigma \sqrt{\frac{2\log(2^{i}/\delta)}{\tau_{i-1}}}\\&\quad \quad -\mathcal{O}\left( \sigma^2 \sqrt{\frac{\log(2^i/\delta)}{\tau_{i-1}}} \right)  -\frac{4\sigma^2}{\tau_{i-1}} \log\left(\frac{|\cF_M|}{\delta}\right).
\end{align*}
Using the definition of $T_M^{(i)}$ and $i^*$, we have
\begin{align*}
    T_m^{(i)} \ge T_M^{(i)} + C_1,
\end{align*}
for some constant $C_1$ (which depends on $C$).

Now, lets consider the realizable classes $m \ge m^*$. We have
\begin{align*}
    T_m^{(i)} \le v^{(i)} + \mathcal{O}\left( \sigma^2 \sqrt{\frac{\log(2^i/\delta)}{\tau_{i-1}}} \right).
\end{align*}
Also, using the fact that $\mathcal{F}_M$ is realizable, and using the lower bound for realizable classes, we have
\begin{align*}
    T_M^{(i)} \ge v^{(i)} - \mathcal{O}\left( \sigma^2 \sqrt{\frac{\log(2^i/\delta)}{\tau_{i-1}}} \right) - \frac{4\sigma^2}{\tau_{i-1} }\log(|\cF_M|/\delta).
\end{align*}
Now, using the definition of $T_M^{(i)}$ and $i^*$, we have
\begin{align*}
    T_m^{(i)} \le T_M^{(i)} + C_1,
\end{align*}
for constant $C_1$.

So, putting the threshold at $\gamma_i = T_M^{(i)} + C_1$, the test succeeds as soon as $i \ge i^*$, with probability at least $1-cM\delta$.

\paragraph{Case 2 -- infinite model classes:}
Here, we combine equations \eqref{eq:realizable}, \eqref{eq:realizable_lower_infinite} and \eqref{eq:infinite-non-realizable} to obtain a threshold $\gamma_i$. Now, suppose for some epoch $i^*$,  satisfies 
\begin{align*}
    2^{i^*} \ge C' \log T \max \left \lbrace \frac{\sigma^4}{(\frac{\Delta^2}{2}-4\eta)^2} \log(1/\delta), \sigma^4 \log\left(\frac{\cN(\cF_M,\alpha,\norm{\cdot}_\infty)}{\delta}\right) \right \rbrace.
\end{align*}
where $C'$ is a sufficiently large universal constant and $\Delta \ge 2\sqrt{2\eta}$. Using a similar argument as above, putting the threshold at $\gamma_i = T_M^{(i)} + C_2$ (where the constant $C_2$ depends on $C'$), the test succeeds as soon as $i \ge i^*$, with probability at least $1-c'M\delta$.

\subsection{Regret Bound of \texttt{ABL} (Proof of Theorem \ref{thm:general})}
Lemma~\ref{lem:gen_infinite} implies that as soon as we reach epoch $i^*$, \texttt{ABL} identifies the model class with high probability, i.e., for each $i \ge i^*$, we have $m^{(i)}=m^*$. However, before that, we do not have any guarantee on the regret performance of \texttt{ABL}. Since at every round the regret can be at most $1$, 
the cumulative regret up until the $i^*$ epoch is upper bounded by $\tau_{i^*-1}$, 
which is at most $\mathcal{O} \left( \sigma^4 \log T \max \left \lbrace \frac{ \log(1/\delta)}{(\frac{\Delta^2}{2}-4\eta)^2} , \log\left(\frac{\cN(\cF_M)}{\delta}\right) \right \rbrace \right)$ if the model classes are infinite, and $\mathcal{O} \left( \sigma^4 \log T \max \left \lbrace \frac{ \log(1/\delta)}{(\frac{\Delta^2}{2}-4\eta)^2} , \log\left(\frac{|\mathcal{F}_{M}|}{\delta}\right) \right \rbrace \right)$, if the model classes are finite. Note that this is the cost we pay for model selection. 

Now, let us bound the regret of \texttt{ABL} from epoch $i^*$ onward. 
Let $R^{\texttt{BL}}(t_i,\delta_i,\cF_{m^{(i)}})$ denote the cumulative regret of \texttt{BL}, when it is run for $k_i$ episodes with confidence level $\delta_i$ for the family $\cF_{m^{(i)}}$. Now, using the result of \cite{russo2013eluder}, we have
\begin{align*}
    R^{\texttt{BL}}(t_i,\delta_i,\cF_m^{(i)}) \le 1+\dim_{\cE}\left(\cF_{m^{(i)}},\frac{1}{t_i}\right) &+4\sqrt{\beta_{t_i}(\cF_{m^{(i)}},\delta_i)\dim_{\cE}\left(\cF_{m^{(i)}},\frac{1}{t_i}\right)t_i }
\end{align*}
with probability at least $1-\delta_i$.
With this and Lemma \ref{lem:gen_infinite}, the regret of \texttt{ABL} after $T$ episodes rounds is given by
\begin{align*}
    R(T) &\leq \tau_{i^*-1} + \sum_{i=i^*}^N R^{\texttt{BL}}(t_i,\delta_i,\cF_{m^{(i)}}) \\
     & \leq \tau_{i^*-1}H + N+ \sum_{i=i^*}^N \dim_{\cE}\left(\cF_{m^*},\frac{1}{t_i}\right) +4 \sum_{i=i^*}^N \sqrt{\beta_{t_i}(\cF_{m^*},\delta_i)\dim_{\cE}\left(\cF_{m^*},\frac{1}{t_i}\right)t_i }.
\end{align*}
The above expression holds with probability at least $1-c'M\delta-2\sum_{i=i^*}^{N}\delta_i$ for infinite model classes, and with probability at least $1-cM\delta-2\sum_{i=i^*}^{N}\delta_i$ for finite model classes. Let us now compute the last term in the above expression. 

% Substituting $\delta_i = \delta/2^i$, we have
% \begin{align*}
%     \sum_{i=i^*}^N H\sqrt{2k_iH\log(1/\delta_i)} &= \sum_{i=i^*}^N H\sqrt{2k_i H \,i \,\log(2/\delta)} \\
%      & \leq H\sqrt{2HN\log(2/\delta)} \sum_{i=1}^N \sqrt{k_i} \\
%      &= \mathcal{O}\left(  H \sqrt{KH N \log(1/\delta)}\right) =  \mathcal{O}\left( H \sqrt{T\log K \log(1/\delta)}\right),
% \end{align*}
 
     We can upper bound the third term in the regret expression as
  \begin{align*}
      \sum_{i=i^*}^N \dim_{\cE}\left(\cF_{m^*},\frac{1}{t_i}\right) \le  N \dim_{\cE}\left(\cF_{m^*},\frac{1}{T}\right) = \cO \left(d_{\cE}(\cF^*) \log T \right), 
  \end{align*}
  where we have used that the total number of epochs $N=\cO(\log T)$.
    Now, notice that, by substituting $\delta_i = \delta/2^i$, we can upper bound $\beta_{t_i}(\cF_{m^*},\delta_i)$ as follows:
  \begin{align*}
      \beta_{t_i}(\cF_{m^*},\delta_i) &=  \mathcal{O}\left(  i \, \log \left( \frac{\cN(\mathcal{F}_{m^*},\frac{1}{t_i },\norm{\cdot}_{\infty})}{\delta}\right) +  \left(1+ \sqrt{i\log \frac{ t_i(
    t_i +1)}{\delta}} \right) \right)\\ & \leq \mathcal{O}\left( \, N \, \log \left( \frac{\cN(\mathcal{F}_{m^*},\frac{1}{T },\norm{\cdot}_{\infty})}{\delta}\right) +  \left(1+ \sqrt{N \log \frac{T}{\delta}} \right) \right) \\
      & \leq \mathcal{O}\left(  \log T \, \log \left( \frac{\cN(\mathcal{F}_{m^*},\frac{1}{T},\norm{\cdot}_{\infty})}{\delta}\right) +   \sqrt{ \log T \log (T/\delta)}  \right)
  \end{align*}
  for infinite model classes, and $\beta_{t_i}(\cF_{m^*},\delta_i) = \cO\left(\log T\log\left(\frac{|\cF_{m^*}|}{\delta} \right) \right)$ for finite model classes.
  With this, the last term in the regret expression can be upper bounded as 
  \begin{align*}
    & \sum_{i=i^*}^N 4\sqrt{\beta_{t_i}(\cF_{m^*},\delta_i)\dim_{\cE}\left(\cF_{m^*},\frac{1}{t_i}\right)t_i }  \\
    & \leq \mathcal{O} \left(\sqrt{  \log T \, \log \left( \frac{\cN(\mathcal{F}_{m^*},\frac{1}{T },\norm{\cdot}_{\infty})}{\delta}\right) + \sqrt{ \log T \log (T/\delta)}  } \,\,\sqrt{\dim_{\cE}\left(\cF_{m^*},\frac{1}{T}\right)} \sum_{i=i^*}^N \sqrt{t_i } \right)\\
    & \leq \mathcal{O} \left(\sqrt{ \log T \, \log \left( \frac{\cN(\mathcal{F}_{m^*},\frac{1}{T},\norm{\cdot}_{\infty})}{\delta}\right) \log\left(\frac{T}{\delta}\right)} \,\,\sqrt{T\,\,\dim_{\cE}\left(\cF_{m^*},\frac{1}{T}\right)}\right)\\
    & = \cO \left(\sqrt{Td_{\cE}(\cF^*)\log (\cN(\cF^*)/\delta)\log T \log(T/\delta)} \right)
  \end{align*}
  for infinite model classes. Similarly, for finite model classes, we can upper bound this term by $\cO \left(\sqrt{Td_{\cE}(\cF^*)\log\left(\frac{|\cF^*|}{\delta} \right)\log T} \right)$. Here, we have used that
\begin{align*}
    \sum_{i=1}^N \sqrt{t_i} & = \sqrt{t_N}\left(1 + \frac{1}{\sqrt{2}} + \frac{1}{2} + \ldots N\text{-th term} \right)\\
    & \leq \sqrt{t_N}\left(1 + \frac{1}{\sqrt{2}} + \frac{1}{2} + ... \right)= \frac{\sqrt{2}}{\sqrt{2} -1} \sqrt{t_N}  \leq \frac{\sqrt{2}}{\sqrt{2} -1} \sqrt{T}.
    \end{align*}
 
 Hence, for infinite model classes, the final regret bound can be written as
\begin{align*}
   R(T) &= \mathcal{O} \left( \sigma^4 (\log T) \max \left \lbrace \frac{ \log(1/\delta)}{(\frac{\Delta^2}{2}-4\eta)^2} , \log\left(\frac{\cN(\mathcal{F}_{M})}{\delta}\right) \right \rbrace \right)\\ &\quad + \cO \left(d_{\cE}(\cF^*) \log T+\sqrt{Td_{\cE}(\cF^*)\log(\cN(\cF^*)/\delta)\log T \log(T/\delta)} \right)~.
\end{align*}
 The above regret bound holds with probability greater than
 \begin{align*}
      1-c'M\delta -\sum_{i=i^*}^{N}\frac{\delta}{2^{i-1}} \ge 1-c'M\delta -\sum_{i\ge 1}\frac{\delta}{2^{i-1}} =1-c'M\delta-2\delta~,
 \end{align*}
 which completes the proof of Theorem \ref{thm:general-rl}.
 
Similarly, for finite model classes, the final regret bound can be written as
\begin{align*}
    R(T) &= \mathcal{O} \left( \sigma^4 \log T \max \left \lbrace \frac{ \log(1/\delta)}{(\frac{\Delta^2}{2}-4\eta)^2} , \log\left(\frac{|\mathcal{F}_{M}|}{\delta}\right) \right \rbrace \right)\\ & \quad+  \cO \left(d_{\cE}(\cF^*) \log T+\sqrt{Td_{\cE}(\cF^*)\log\left(\frac{|\cF^*|}{\delta} \right)\log T} \right),
\end{align*}
which holds with probability greater than $1-cM\delta-2\delta$.
\section{Proof for MDPs}
\label{app:gen-rl}

\subsection{Confidence Sets in \texttt{UCRL-VTR}}
We first describe how the confidence sets are constructed in \texttt{UCRL-VTR}. Note that the procedure is similar to that done in \cite{ayoub2020model}, but with a slight difference. Specifically, we define the confidence width as a function of complexity of the transition family $\cP$ on which $P^*$ lies; rather than the complexity of a value-dependent function class induced by $\cP$ (as done in \cite{ayoub2020model}). We emphasize that this small change makes the model selection procedure easier to understand without any effect on the regret.

Let us define, for any two transition kernels $P,P' \in \cP$ and any episode $k \ge 1$, the following
\begin{align*}
    \cL_k(P) &= \sum_{j=1}^{k}\sum_{h=1}^{H} \left(V_{h+1}^j(s_{h+1}^j)-(PV_{h+1}^j)(s_h^j,a_h^j) \right)^2,\\ \cL_k(P,P') &= \sum_{j=1}^{k}\sum_{h=1}^{H} \left((PV_{h+1}^j)(s_h^j,a_h^j)-( P'V_{h+1}^j)(s_h^j,a_h^j) \right)^2.
\end{align*}
Then, the confidence set at the end of episode $k$ is constructed as 
\begin{align*}
    \cB_k = \left\lbrace P \in \cP \,\big |\, \cL_k(P,\widehat P_k) \le \beta_{k}(\delta)\right \rbrace,
\end{align*}
where $\widehat P_k =\argmin_{P \in \cP}\cL_k(\cP)$ denotes an estimate of $P^*$ after $k$ episodes. The confidence width $\beta_k(\delta)\equiv \beta_{k}(\cP,\delta)$ is set as
\begin{align*}
    \beta_k(\delta):=
     \begin{cases}
     8H^2\log\left(\frac{|\cP|}{\delta}\right) & \text{if}\; \cP\; \text{is finite,}\\
     8H^2\log\left(\frac{2\cN\left(\cP,\frac{1}{k H},\norm{\cdot}_{\infty,1}\right)}{\delta}\right) +  4 H^2 \left(2+\sqrt{2 \log \left(\frac{4kH(kH+1)}{\delta} \right)} \right) & \text{if}\; \cP \; \text{is infinite.}
    \end{cases}
\end{align*}

\begin{lemma}[Concentration of $P^*$]\label{lem:conc-ucrlvtr}
For any $\delta \in (0,1]$, with probability at least $1-\delta$, uniformly over all episodes $k \ge 1$, we have $P^* \in \cB_k$.
\end{lemma}
\begin{proof}
First, we define, for any fixed $P \in \cP$ and $(h,k) \in [H] \times [K]$, the quantity
\begin{align*}
    Z_h^{k,P} := 2 \left(( P^* V_{h+1}^k)(s_h^k,a_h^k)-(P V_{h+1}^k)(s_h^k,a_h^k) \right)\left( V_{h+1}^k(s_{h+1}^k)-( P^* V_{h+1}^k)(s_h^k,a_h^k)\right).
\end{align*}
Then, we have
\begin{align}
    \cL_k(\widehat P_k) = \cL_k(P^*) +\cL_k(P^*,\widehat P_{k}) + \sum_{j=1}^{k}\sum_{h=1}^{H}Z_h^{j,\widehat P_k}~.
    \label{eq:break-up-rl}
\end{align}
Using the notation $y_h^k=V_{h+1}^k(s_{h+1}^k)$, we can rewrite $Z_h^{k,P}$ as
\begin{align*}
    Z_h^{k,P} := 2 \left(( P^* V_{h+1}^k)(s_h^k,a_h^k)-(P V_{h+1}^k)(s_h^k,a_h^k) \right)\left( y_h^k-\E[y_h^k|\cG_{h-1}^k]\right),
\end{align*}
where $\cG_{h-1}^k$ denotes the $\sigma$-field summarising all the information available just before $s_{h+1}^k$ is observed.
Note that $Z_h^{k,P}$ is $\cG_h^k$-measurable,  Moreover, since $V_{h+1}^k \in [0,H]$, $Z_h^{k,P}$ is $2H |( P^* V_{h+1}^k)(s_h^k,a_h^k)-(P V_{h+1}^k)(s_h^k,a_h^k)|$-sub-Gaussian conditioned on $\cG_{h-1}^k$. Therefore, for any $\lambda < 0$, with probability at least $1-\delta$, we have
\begin{align*}
    \forall k \ge 1,\quad\sum_{j=1}^{k}\sum_{h=1}^{H} Z_h^{j,P} \ge \frac{1}{\lambda}\log(1/\delta)+ \frac{\lambda}{2}\cdot 4H^2 \sum_{j=1}^{k}\sum_{h=1}^{H}\left(( P^* V_{h+1}^j)(s_h^j,a_h^j)-(P V_{h+1}^j)(s_h^j,a_h^j) \right)^2~.
\end{align*}
Setting $\lambda = -1/(4H^2)$, we obtain for any fixed $P \in \cP$, the following:
\begin{align}
    \forall k \ge 1,\quad\sum_{j=1}^{k}\sum_{h=1}^{H} Z_h^{j,P} \ge -4H^2\log\left(\frac{1}{\delta}\right)- \frac{1}{2} \cL_k(P^*,P)~.
    \label{eq:subg-conc-rl}
\end{align}
with probability at least $1-\delta$. We consider both the cases -- when $\cP$ is finite and when $\cP$ is infinite.

\paragraph{Case 1 -- finite $\cP$:}
We take a union bound over all $P \in \cP$ in \eqref{eq:subg-conc-rl} to obtain that
\begin{align}
    \forall k \ge 1,\;\; \forall P \in \cP,\quad\sum_{j=1}^{k}\sum_{h=1}^{H} Z_h^{j,P} \ge -4H^2\log\left(\frac{|\cP|}{\delta}\right)- \frac{1}{2} \cL_k(P^*,P)
    \label{eq:finite-rl}
\end{align}
with probability at least $1-\delta$. By construction, $\widehat P_k \in \cP$ and $\cL_k(\widehat P_k) \le \cL_k(P^*)$. Therefore, from \eqref{eq:break-up-rl}, we have
\begin{align*}
   \forall k \ge 1,\;\;  \cL_k(P^*,\widehat P_k) \le 8H^2\log\left(\frac{|\cP|}{\delta}\right)
\end{align*}
with probability at least $1-\delta$, which proves the result for finite $\cP$.

\paragraph{Case 2 -- infinite $\cP$:}

Fix some $\alpha > 0$. Let $\cP^\alpha$ denotes an $(\alpha,\norm{\cdot}_{\infty,1})$ cover of $\cP$, i.e., for any $P \in \cP$, there exists an $P^\alpha$ in $\cP^\alpha$ such that $\norm{P^\alpha - P}_{\infty,1}:=\sup_{s,a}\int_{\cS}|P^\alpha(s'|s,a)-P(s'|s,a)|ds' \le \alpha$. Now, we take a union bound over all $P^\alpha \in \cP^\alpha$ in \eqref{eq:subg-conc-rl} to obtain that
\begin{align*}
    \forall k \ge 1,\;\; \forall P^\alpha \in \cP^\alpha,\quad\sum_{j=1}^{k}\sum_{h=1}^{H} Z_h^{j,P^\alpha} \ge -4H^2\log\left(\frac{|\cP^\alpha|}{\delta}\right)- \frac{1}{2} \cL_k(P^*,P^\alpha)~.
\end{align*}
with probability at least $1-\delta$, and thus, in turn, 
\begin{align}
    \forall k \ge 1,\;\; \forall P\in \cP,\quad\sum_{j=1}^{k}\sum_{h=1}^{H} Z_h^{j,P} \ge -4H^2\log\left(\frac{|\cP^\alpha|}{\delta}\right)- \frac{1}{2} \cL_k(P^*,P)+\zeta_k^{\alpha}(P).
    \label{eq:comb-one-conf-rl}
\end{align}
with probability at least $1-\delta$,
where $\zeta_k^{\alpha}(P)$ denotes the discretization error:
\begin{align*}
    &\zeta_k^{\alpha}(P)\\ =& \sum_{j=1}^{k}\sum_{h=1}^{H} \left(Z_h^{j,P}-  Z_h^{j,P^\alpha}\right) + \frac{1}{2}\cL_k(P^*,P) - \frac{1}{2}\cL_k(P^*,P^\alpha)\\
    =& \sum_{j=1}^{k}\sum_{h=1}^{H} \bigg(2y_h^j \left(( P^\alpha V_{h+1}^j)(s_h^j,a_h^j)-(P V_{h+1}^j)(s_h^j,a_h^j) \right) + \frac{1}{2}(P V_{h+1}^j)^2(s_h^j,a_h^j) \\
    & \qquad \qquad - \frac{1}{2}(P^\alpha V_{h+1}^j)^2(s_h^j,a_h^j)\bigg).
\end{align*}
Since $\norm{P-P^\alpha}_{\infty,1} \le \alpha$ and $\norm{V_{h+1}^k}_\infty \le H$, we have 
\begin{align*}
 \left |(P^\alpha V_{h+1}^k)(s,a)-(PV_{h+1}^k)(s,a)\right| \le \alpha H~,   
\end{align*}
which further yields
\begin{align*}
    \left|(P^\alpha V_{h+1}^k)^2(s,a)- (P V_{h+1}^k)^2(s,a)\right| &\le \max_{|\xi| \le \alpha H} \left | \left((PV_{h+1}^k)(s,a)+\xi\right)^2-(PV_{h+1}^k)(s,a)^2 \right|\\ &\le 2 \alpha H^2 +\alpha^2H^2~.
\end{align*}
Therefore, we can upper bound the discretization error as
\begin{align*}
    |\zeta_k^{\alpha}(P)| &\le 2 \alpha H \sum_{j=1}^{k}\sum_{h=1}^{H}|y_h^j| + \sum_{j=1}^{k}\sum_{h=1}^{H} \left(\alpha H^2 + \frac{\alpha^2H^2}{2}\right)\\ &\le 2 \alpha H \sum_{j=1}^{k}\sum_{h=1}^{H}|y_h^j-\E[y_h^j|\cG_{h-1}^j]| + \sum_{j=1}^{k}\sum_{h=1}^{H} \left(3\alpha H^2 + \frac{\alpha^2H^2}{2}\right)~.
\end{align*}
Since $y_h^k-\E[y_h^k|\cG_{h-1}^k]$ is $H$-sub-Gaussian conditioned on $\cG_{h-1}^k$, we have 
\begin{align*}
    \forall k \ge 1,\forall h \in [H],\quad |y_h^k-\E[y_h^k|\cG_{h-1}^k]| \le H\sqrt{2 \log \left(\frac{2kH(kH+1)}{\delta} \right)}
\end{align*}
with probability at least $1-\delta$.
Therefore, with probability at least $1-\delta$, the discretization error is bounded for all episodes $k \ge 1$ as
\begin{align*}
    |\zeta_k^{\alpha}(P)| &\le kH \left( 2\alpha H^2 \sqrt{2 \log \left(\frac{2kH(kH+1)}{\delta} \right)} + 3\alpha H^2+\frac{\alpha^2H^2}{2}\right)\\
    & \le \alpha kH \left( 2 H^2 \sqrt{2 \log \left(\frac{2kH(kH+1)}{\delta} \right)} + 4H^2\right),
\end{align*}
where the last step holds for any $\alpha \leq 1$. Therefore, from \eqref{eq:comb-one-conf-rl}, we have
\begin{align}
   \forall k \ge 1,\;\; \forall P\in \cP,\quad\sum_{j=1}^{k}\sum_{h=1}^{H} Z_h^{j,P} &\ge -4H^2\log\left(\frac{|\cP^\alpha|}{\delta}\right)- \frac{1}{2} \cL_k(P^*,P)\nonumber\\ &  \quad \quad- \alpha kH \left( 2 H^2 \sqrt{2 \log \left(\frac{2kH(kH+1)}{\delta} \right)} + 4H^2\right),
   \label{eq:infinite-rl}
\end{align}
with probability at least $1-2\delta$.
Now, setting $\alpha=\frac{1}{k H}$, we obtain, from \eqref{eq:break-up-rl}, that
\begin{align*}
    \forall k \ge 1,\; \cL_k(P^*,\hat P_{k}) \le 8H^2\log\left(\frac{2\cN\left(\cP,\frac{1}{k H},\norm{\cdot}_{\infty,1}\right)}{\delta}\right) +  4 H^2 \left(2+\sqrt{2 \log \left(\frac{4kH(kH+1)}{\delta} \right)} \right)
\end{align*}
with probability at least $1-\delta$, which proves the result for infinite $\cP$.
\end{proof}

\subsection{Model Selection in \texttt{ARL}}

First, we find concentration bounds on the test statistics $T_m^{(i)}$ for all epochs $i \ge 1$ and class indexes $m \in [M]$, which are crucial to prove the model selection guarantee of \texttt{ARL}.

\paragraph{1. Realizable model classes:} Fix a class index $m \ge m^*$. In this case, the true model $P^* \in \cP_m$. Therefore, the we can upper bound the empirical risk at epoch $i$ as
\begin{align*}
    T_m^{(i)}&=\frac{1}{\tau_{i-1}H}\sum_{k=1}^{\tau_{i-1}}\sum_{h=1}^{H} \left(V_{h+1}^k(s_{h+1}^k)-( \widehat{P}^{(i)}_m V_{h+1}^k)(s_h^k,a_h^k) \right)^2\\ &\le \frac{1}{\tau_{i-1}H}\sum_{k=1}^{\tau_{i-1}}\sum_{h=1}^{H} \left(V_{h+1}^k(s_{h+1}^k)-( P^* V_{h+1}^k)(s_h^k,a_h^k) \right)^2\\
    &= \frac{1}{\tau_{i-1}H}\sum_{k=1}^{\tau_{i-1}}\sum_{h=1}^{H} \left( y_h^k - \E[y_h^k|\cG_{h-1}^k]\right)^2.
\end{align*}
Now, we define the random variable $m_h^k := \left( y_h^k - \E[y_h^k|\cG_{h-1}^k]\right)^2$. Note that $\E[m_h^k|\cG_{h-1}^k] = \Var[y_h^k|\cG_{h-1}^k]$. Moreover, $(m_h^k-\E[m_h^k|\cG_{h-1}^k])_{k,h}$ is a martingale difference sequence adapted to the filtration $\cG_h^k$, with absolute values $|m_h^k-\E[m_h^k|\cG_{h-1}^k]| \le H^2$ for all $k,h$. Therefore, by the  Azuma-Hoeffding inequality, with probability at least $1-\delta/2^{i}$,
\begin{align*}
    \sum_{k=1}^{\tau_{i-1}}\sum_{h=1}^H m_h^k \le  \sum_{k=1}^{\tau_{i-1}}\sum_{h=1}^H \E[m_h^k|\cG_{h-1}^k] + H^2\sqrt{2\tau_{i-1}H\log(2^{i}/\delta)}.
\end{align*}
Now, using a union bound, along with the definition of $T^{(i)}_m$, with probability at least $1-\delta$, we have
\begin{align}
\forall i \ge 1,\quad  T_m^{(i)} \le v^{(i)}  + H^{3/2}\sqrt{\frac{2\log(2^{i}/\delta)}{\tau_{i-1}}}, 
  \label{eq:realizable-rl}  
\end{align}
where $v^{(i)}=\frac{1}{\tau_{i-1}H}\sum_{k=1}^{\tau_{i-1}}\sum_{h=1}^{H}\Var[y_h^k|\cG_{h-1}^k]$.

Let us now find a lower bound of $T^{(i)}_m$. We can decompose it as  $T_m^{(i)}=T^{(i)}_{m,1}+T^{(i)}_{m,2}+T^{(i)}_{m,3}$, where
\begin{align*}
    T^{(i)}_{m,1}  &= \frac{1}{\tau_{i-1}H}\sum_{k=1}^{\tau_{i-1}}\sum_{h=1}^{H}  \left(V_{h+1}^k(s_{h+1}^k)-( P^* V_{h+1}^k)(s_h^k,a_h^k) \right)^2,\\
    T^{(i)}_{m,2} & = \frac{1}{\tau_{i-1}H}\sum_{k=1}^{\tau_{i-1}}\sum_{h=1}^{H} \left(( P^* V_{h+1}^k)(s_h^k,a_h^k)-( \hat P^{(i)}_m V_{h+1}^k)(s_h^k,a_h^k) \right)^2,\\
    T^{(i)}_{m,3} &= \frac{1}{\tau_{i-1}H}\sum_{k=1}^{\tau_{i-1}}\sum_{h=1}^{H} 2 \left(( P^* V_{h+1}^k)(s_h^k,a_h^k)-( \hat P^{(i)}_m V_{h+1}^k)(s_h^k,a_h^k) \right)\left( V_{h+1}^k(s_{h+1}^k) - (P^*V_{h+1}^k)(s_h^k,a_h^k)\right).
\end{align*}
First, using a similar argument as in \eqref{eq:realizable-rl}, with probability at least $1-\delta$, we obtain
\begin{align*}
    \forall i \ge 1, \quad T^{(i)}_{m,1} \ge v^{(i)} - H^{3/2}\sqrt{\frac{2\log(2^{i}/\delta)}{\tau_{i-1}}}.
\end{align*}
Next, trivially, we have \begin{align*}
\forall i \ge 1, \quad    T^{(i)}_{m,2} \ge 0.
\end{align*}
Now, we turn to bound the term $T^{(i)}_{m,3}$. We consider both the cases -- when $\cP$ is finite and when $\cP$ is infinite.

\paragraph{Case 1 -- finite model classes:}
Note that $\hat P_m^{(i)} \in \cP_m$. Then, from \eqref{eq:finite-rl}, we have
\begin{align*}
     \forall i \ge 1,\quad T_{m,3}^{(i)} \ge 
    -\frac{4H}{\tau_{i-1}} \log\left(\frac{|\cP_m|}{\delta}\right) -\frac{1}{2} T_{m,2}^{(i)}
\end{align*}
with probability at least $1-\delta$. Now, combining all the three terms together and using a union bound, we obtain 
\begin{align}
   \forall i \ge 1,\quad T_m^{(i)} \ge v^{(i)} -  H^{3/2}\sqrt{\frac{2\log(2^{i}/\delta)}{\tau_{i-1}}}  -\frac{4H}{\tau_{i-1}} \log\left(\frac{|\cP_m|}{\delta}\right).
   \label{eq:finite-realizable-lower-rl}
 \end{align}
with probability at least $1-c_1\delta$ (where $c_1$ is a positive universal constant).

\paragraph{Case 2 -- infinite model classes:}
We follow a similar procedure, using \eqref{eq:infinite-rl}, to obtain
\begin{align}
     \forall i \ge 1,\quad T_{m}^{(i)} &\ge v^{(i)} - H^{3/2}\sqrt{\frac{2\log(2^{i}/\delta)}{\tau_{i-1}}}
    -\frac{4H}{\tau_{i-1}} \log\left(\frac{\cN(\cP_m,\alpha,\norm{\cdot}_{\infty,1})}{\delta}\right)\nonumber\\ &\quad\quad- \alpha \left( 2 H^2 \sqrt{2 \log \left(\frac{2\tau_{i-1}H(\tau_{i-1}H+1)}{\delta} \right)} + 4H^2\right)
    \label{eq:infinite-realizable-lower-rl}
\end{align}
with probability at least $1-c_2\delta$ (where $c_2$ is a positive universal constant).

% We have
% \begin{align*}
%      T_m^{(i)} = \frac{1}{\tau_{i-1}H}\sum_{k=1}^{\tau_{i-1}}\sum_{h=1}^{H} \left(V_{h+1}^k(s_{h+1}^k)-( \widehat{P}^{(i)}_m V_{h+1}^k)(s_h^k,a_h^k) \right)^2.
% \end{align*}
% We write
% \begin{align*}
%     & \frac{1}{\tau_{i-1}H}\sum_{k=1}^{\tau_{i-1}}\sum_{h=1}^{H} \left(V_{h+1}^k(s_{h+1}^k)-( \widehat{P}^{(i)}_m V_{h+1}^k)(s_h^k,a_h^k) \right)^2 - \frac{1}{\tau_{i-1}H}\sum_{k=1}^{\tau_{i-1}}\sum_{h=1}^{H} \left(V_{h+1}^k(s_{h+1}^k)-( \widehat{P}^* V_{h+1}^k)(s_h^k,a_h^k) \right)^2 \\
%     & = \frac{1}{\tau_{i-1}H}\sum_{k=1}^{\tau_{i-1}}\sum_{h=1}^{H} \left[ (\widehat{P}^* V_{h+1}^k(s_h^k,a_h^k) - \widehat{P}^{(i)}_m V_{h+1}^k)(s_h^k,a_h^k))(2V_{h+1}^k(s_{h+1}^k) - (\widehat{P}^{(i)}_m V_{h+1}^k)(s_h^k,a_h^k) - (\widehat{P}^* V_{h+1}^k)(s_h^k,a_h^k) \right]
% \end{align*}
% Taking expectations (w.r.t the true model $P^*$), we obtain
% \begin{align*}
%     \mathbb{E}T^{(i)}_m \geq \sigma^2,
% \end{align*}
% and hence, using similar martingale difference construction, we obtain
% \begin{align*}
%     T^{(i)}_m \geq \sigma^2 - H^{3/2}\sqrt{\frac{2\log(2^{i}/\delta)}{\tau_{i-1}}}
% \end{align*}
% with probability at least $1-\delta$.

\paragraph{2. Non-realizable model classes:} Fix a class index $m < m^*$. In this case, the true model $P^* \notin \cP_m$. We can decompose the empirical risk at epoch $i$ as $T_m^{(i)}=T^{(i)}_{m,1}+T^{(i)}_{m,2}+T^{(i)}_{m,3}$, where
\begin{align*}
    S^{(i)}_{m,1}  &= \frac{1}{\tau_{i-1}H}\sum_{k=1}^{\tau_{i-1}}\sum_{h=1}^{H}  \left(V_{h+1}^k(s_{h+1}^k)-( P^* V_{h+1}^k)(\tilde s_h^k,\tilde a_h^k) \right)^2,\\
    S^{(i)}_{m,2} & = \frac{1}{\tau_{i-1}H}\sum_{k=1}^{\tau_{i-1}}\sum_{h=1}^{H} \left(( P^* V_{h+1}^k)(\tilde s_h^k,\tilde a_h^k)-( \hat P^{(i)}_m V_{h+1}^k)(s_h^k,a_h^k) \right)^2,\\
    S^{(i)}_{m,3} &= \frac{1}{\tau_{i-1}H}\sum_{k=1}^{\tau_{i-1}}\sum_{h=1}^{H} 2 \left(( P^* V_{h+1}^k)(\tilde s_h^k,\tilde a_h^k)-( \hat P^{(i)}_m V_{h+1}^k)(s_h^k,a_h^k) \right)\left( V_{h+1}^k(s_{h+1}^k) - (P^*V_{h+1}^k)(\tilde s_h^k,\tilde a_h^k)\right),
\end{align*}
where $(\tilde s_h^k,\tilde a_h^k)$ satisfies
\begin{align*}
    D^*\left((s_h^k,a_h^k),(\tilde s_h^k,\tilde a_h^k)\right)=\abs{(P^*V_{h+1}^k)( s_h^k, a_h^k)-(P^*V_{h+1}^k)(\tilde s_h^k,\tilde a_h^k)}\leq \eta.
\end{align*}
To bound the first term, note that
\begin{align*}
   &\left(V_{h+1}^k(s_{h+1}^k)-( P^* V_{h+1}^k)(\tilde s_h^k,\tilde a_h^k) \right)^2\\ &\ge \left(V_{h+1}^k(s_{h+1}^k)-( P^* V_{h+1}^k)( s_h^k, a_h^k) \right)^2 \\ &+2\left((P^*V_{h+1}^k)( s_h^k, a_h^k)-(P^*V_{h+1}^k)(\tilde s_h^k,\tilde a_h^k)\right)\left(V_{h+1}^k(s_{h+1}^k)-( P^* V_{h+1}^k)( s_h^k, a_h^k) \right)\\
   & = m_h^k+2\left((P^*V_{h+1}^k)( s_h^k, a_h^k)-(P^*V_{h+1}^k)(\tilde s_h^k,\tilde a_h^k)\right) \left( y_h^k - \E[y_h^k|\cG_{h-1}^k]\right).
\end{align*}
Define $Y_h^k=2\left((P^*V_{h+1}^k)( s_h^k, a_h^k)-(P^*V_{h+1}^k)(\tilde s_h^k,\tilde a_h^k)\right) \left( y_h^k - \E[y_h^k|\cG_{h-1}^k]\right)$. Note that $Y_h^k$ is zero-mean $2H^2$ sub-Gaussian conditioned on $\cG_{h-1}^{k}$. Therefore, by Azuma-Hoeffding inequality and using a similar argument as in \eqref{eq:realizable-rl}, with probability at least $1-\delta$, we obtain
\begin{align*}
    \forall i \ge 1, \quad S^{(i)}_{m,1} \ge v^{(i)} - O\left(H^{3/2}\sqrt{\frac{2\log(2^{i}/\delta)}{\tau_{i-1}}}\right).
\end{align*}
The second term is bounded by Assumption \ref{ass:sep-rl} as 
\begin{align*}
\forall i \ge 1, \quad    S^{(i)}_{m,2} \ge \Delta^2~.
\end{align*}
Now, we turn to bound the third term $T^{(i)}_{m,3}$. Note that 
\begin{align*}
& 2\left(( P^* V_{h+1}^k)(\tilde s_h^k,\tilde a_h^k)-( \hat P^{(i)}_m V_{h+1}^k)(s_h^k,a_h^k) \right)\left( V_{h+1}^k(s_{h+1}^k) - (P^*V_{h+1}^k)(\tilde s_h^k,\tilde a_h^k)\right)\\
&= 2 \left(( P^* V_{h+1}^k)(\tilde s_h^k,\tilde a_h^k)-( \hat P^{(i)}_m V_{h+1}^k)(s_h^k,a_h^k) \right)\left( V_{h+1}^k(s_{h+1}^k) - (P^*V_{h+1}^k)( s_h^k, a_h^k)\right)\\
&\quad\quad\quad -2 \left(( P^* V_{h+1}^k)(\tilde s_h^k,\tilde a_h^k)-( \hat P^{(i)}_m V_{h+1}^k)(s_h^k,a_h^k) \right)\left((P^*V_{h+1}^k)( \tilde s_h^k, \tilde a_h^k)-(P^*V_{h+1}^k)( s_h^k, a_h^k)\right)\\
& \geq 2 \left(( P^* V_{h+1}^k)(\tilde s_h^k,\tilde a_h^k)-( \hat P^{(i)}_m V_{h+1}^k)(s_h^k,a_h^k) \right)\left( V_{h+1}^k(s_{h+1}^k) - (P^*V_{h+1}^k)( s_h^k, a_h^k)\right) - 2H\eta
\end{align*}

We now consider both the cases -- when $\cP$ is finite and when $\cP$ is infinite.

\paragraph{Case 1 -- finite model classes:}
Note that $\hat P_m^{(i)} \in \cP_m$. Then, from \eqref{eq:finite-rl}, we have
\begin{align*}
     \forall i \ge 1,\quad S_{m,3}^{(i)} \ge 
    -\frac{4H}{\tau_{i-1}} \log\left(\frac{|\cP_m|}{\delta}\right) -\frac{1}{2} T_{m,2}^{(i)} -2H\eta
\end{align*}
with probability at least $1-\delta$. Now, combining all the three terms together and using a union bound, we obtain 
\begin{align}
   \forall i \ge 1,\quad T_m^{(i)} \ge v^{(i)} + (\frac{1}{2}\Delta^2-2H\eta) - O\left(H^{3/2}\sqrt{\frac{2\log(2^{i}/\delta)}{\tau_{i-1}}}\right)  -\frac{4H}{\tau_{i-1}} \log\left(\frac{|\cP_m|}{\delta}\right) .
   \label{eq:finite-non-realizable-rl}
 \end{align}
with probability at least $1-c_3\delta$.

\paragraph{Case 2 -- infinite model classes:}
We follow a similar procedure, using \eqref{eq:infinite-rl}, to obtain 
\begin{align}
     \forall i \ge 1,\quad T_{m}^{(i)} &\ge v^{(i)} + (\frac{1}{2}\Delta^2-2H\eta) - O\left( H^{3/2}\sqrt{\frac{2\log(2^{i}/\delta)}{\tau_{i-1}}}\right) \\
     &\qquad-\frac{4H}{\tau_{i-1}} \log\left(\frac{\cN(\cP_m,\alpha,\norm{\cdot}_{\infty,1})}{\delta}\right)\nonumber\\ &\quad\quad- \alpha \left( 2 H^2 \sqrt{2 \log \left(\frac{2\tau_{i-1}H(\tau_{i-1}H+1)}{\delta} \right)} + 4H^2\right).
    \label{eq:infinite-non-realizable-rl}
\end{align}
with probability at least $1-c_4\delta$. 

\subsubsection{Proof of Lemma \ref{lem:gen_infinite-rl}}
We are now ready to prove Lemma \ref{lem:gen_infinite-rl}, which presents the model selection Guarantee of \texttt{ARL} for infinite model classes $\lbrace \cP_m \rbrace_{m \in [M]}$. Here, at the same time, we prove a similar (and simpler) result for finite model classes also. 

First, note that we consider doubling epochs $k_i=2^i$, which implies $\tau_{i-1} = \sum_{j=1}^{i-1}k_j = \sum_{j=1}^{i-1} 2^j = 2^i -1$. With this, the number of epochs is given by $N=\lceil \log_2(K+1) -1 \rceil = \mathcal{O}(\log K)$. 

\paragraph{Case 1 -- finite model classes:}

Let us now consider finite model classes $\lbrace \cP_m \rbrace_{m \in [M]}$.

First, we combine \eqref{eq:realizable-rl}, \eqref{eq:finite-realizable-lower-rl}
 and \eqref{eq:finite-non-realizable-rl}, and take a union bound over all $m \in [M]$ to obtain
\begin{align*}
    &\forall m \ge m^*,\;\forall i \ge 1, \, v^{(i)}  - O\left(H^{3/2}\sqrt{\frac{2N\log(2/\delta)}{\tau_{i-1}}}\right)-\frac{4H}{\tau_{i-1}} \log\left(\frac{|\cP_M|}{\delta}\right) \le T_m^{(i)}\\&\quad\quad\quad\quad \le v^{(i)}  + O\left( H^{3/2}\sqrt{\frac{2N\log(2/\delta)}{\tau_{i-1}}}\right),\\
    &\forall m \le m^*-1,\;\forall i \ge 1,\quad T_m^{(i)} \ge v^{(i)} + (\frac{1}{2}\Delta^2-2H\eta)\\\quad\quad\quad & - O\left(H^{3/2}\sqrt{\frac{2N\log(2/\delta)}{\tau_{i-1}}}  -\frac{4H}{\tau_{i-1}} \log\left(\frac{|\cP_M|}{\delta}\right) \right)
\end{align*}
with probability at least $1-cM\delta$, where we have used that $\log(2^i/\delta) \le N\log(2/\delta)$ for all $i$ and $|\cP_m| \le |\cP_M|$ for all $m$.
%  With this, equations~\eqref{eqn:realizable-finite} and \eqref{eqn:nonrealizable-finite} can be rewritten as
% \begin{align*}
%     \forall m \ge m^*,\;\forall i \ge 1,\quad T_m^{(i)} &\le \sigma^2  + H^{3/2}\sqrt{\frac{2 N\log(2/\delta)}{\tau_{i-1}}} \; \text{and}\\
%     \forall m \le m^*-1,\;\forall i \ge 1,\quad T_m^{(i)} &\ge \sigma^2 + \frac{1}{2}\Delta - H^{3/2}\sqrt{\frac{2 N\log(2/\delta)}{\tau_{i-1}}}  -\frac{4H}{\tau_{i-1}} \log\left(\frac{|\cP_M|}{\delta}\right) 
% \end{align*}
Now, suppose for some epoch $i^*$, satisfies 
\begin{align*}
    2^{i^*} \geq C'\log K\max \left \lbrace \frac{2H^3 }{(\frac{1}{2}\Delta^2-2H\eta)^2} \log(2/\delta), 4H \log\left(\frac{|\mathcal{P}_M|}{\delta}\right) \right \rbrace,
\end{align*}
where $C$ is a sufficiently large universal constant and $\Delta \ge 2\sqrt{H\eta}$. 

Then, using the upper bound on $T_M^{(i)}$, we have 
\begin{align*}
    \forall m \le m^*-1,\;\forall i \ge i^*,\quad T_m^{(i)} \ge T_M^{(i)} + C_2
\end{align*}
for some sufficiently small constant $C_2$ (which depends on $C'$).
Similarly, using the lower bound on $T_M^{(i)}$, we have 
\begin{align*}
    \forall m \ge m^*,\;\forall i \ge i^*,\quad T_m^{(i)} \le T_M^{(i)} + C_2.
\end{align*}

% Then, we have
% \begin{align*}
%     &\forall m \ge m^*,\;\forall i \ge i^*,\quad \sigma^2 - \frac{c_0}{2^{i/2}} \le T_m^{(i)} \le \sigma^2 + \frac{c_0}{2^{i/2}} \;\; \text{and}\\
%     &\forall m \le m^*-1,\;\forall i \ge i^*,\quad T_m^{(i)} \ge \sigma^2 + (\frac{1}{2} \Delta^2-2H\eta)  - \frac{c_1}{2^{i/2}}
% \end{align*}
% with probability at least $1-3M\delta$.
% Note that with the chosen threshold $\gamma_i=T_M^{(i)}+\frac{\sqrt{i}}{2^{i/2}}$, for all $i \geq i^*$
% \begin{align*}
%   \sigma^2 - \frac{c_0}{2^{i/2}} + \frac{\sqrt{i}}{2^{i/2}} \leq \gamma_i &\leq \sigma^2 +\frac{c_0}{2^{i/2}} + \frac{\sqrt{i}}{2^{1/2}} 
%     \leq\sigma^2 + (\frac{1}{2}\Delta^2-2H\eta) -\frac{c_1}{2^{i/2}}.
% \end{align*}
% where the last inequality comes from the choice of $2^{i^*}$. With this, we have
% \begin{align*}
%     \forall m \ge m^*,\;\forall i \ge i^*,\quad T_m^{(i)} &\le \gamma_i \; \text{and}\\
%     \forall m \le m^*-1,\;\forall i \ge i^*,\quad T_m^{(i)} &\ge \gamma_i
% \end{align*}
% with probability at least $1-3M\delta$. 

The above equation implies that $m^{(i)}=m^*$ for all epochs $i \ge i^*$ with probability at least $1-cM\delta$.

\paragraph{Case 2 -- infinite model classes:} 

Now, we focus on infinite model classes $\lbrace \cP_m \rbrace_{m \in [M]}$.

First, we combine \eqref{eq:realizable-rl}, \eqref{eq:infinite-realizable-lower-rl},
 and \eqref{eq:infinite-non-realizable-rl}, and take a union bound over all $m \in [M]$ to obtain
\begin{align*}
    & \forall m \ge m^*,\;\forall i \ge 1,v^{(i)}  - O\left( H^{3/2}\sqrt{\frac{2N\log(2/\delta)}{\tau_{i-1}}}\right)  -\frac{4H}{\tau_{i-1}} \log\left(\frac{\cN(\cP_M,\alpha,\norm{\cdot}_{\infty,1})}{\delta}\right)\\ &\quad\quad- \alpha \left( 2 H^2 \sqrt{2 \log \left(\frac{2KH(KH+1)}{\delta} \right)} + 4H^2\right) \leq  T_m^{(i)} \le v^{(i)}  + O\left(H^{3/2}\sqrt{\frac{2N\log(2/\delta)}{\tau_{i-1}}}\right),\\
    & \forall m \le m^*-1,\;\forall i \ge 1,\quad T_m^{(i)} \ge   v^{(i)}+ (\frac{1}{2}\Delta^2-2H\eta) - O\left(H^{3/2}\sqrt{\frac{2N\log(2/\delta)}{\tau_{i-1}}}\right)  \\&\quad\quad-\frac{4H}{\tau_{i-1}} \log\left(\frac{\cN(\cP_M,\alpha,\norm{\cdot}_{\infty,1})}{\delta}\right)\nonumber - \alpha \left( 2 H^2 \sqrt{2 \log \left(\frac{2K H(KH+1)}{\delta} \right)} + 4H^2\right)
\end{align*}
with probability at least $1-cM\delta$. Suppose for some epoch $i^*$, we have,
\begin{align*}
    2^{i^*} \geq C'\log K \max \left \lbrace \frac{H^3 }{(\frac{1}{2}\Delta^2-2H\eta)^2} \log(2/\delta), 4H \log\left(\frac{\cN(\cP_M,\alpha,\norm{\cdot}_{\infty,1})}{\delta}\right) \right \rbrace
\end{align*}
where $C'>1$ is a sufficiently large universal constant. Then, with the choice of threshold $\gamma_i$, and doing the same calculation as above, we obtain
\begin{align*}
    \forall m \ge m^*,\;\forall i \ge i^*,\quad T_m^{(i)} &\le \gamma_i \; \text{and}\\
    \forall m \le m^*-1,\;\forall i \ge i^*,\quad T_m^{(i)} &\ge \gamma_i
\end{align*}
with probability at least $1-c'M\delta$. The above equation implies that $m^{(i)}=m^*$ for all epochs $i \ge i^*$, proving the result.

\subsection{Regret Bound of \texttt{ARL} (Proof of Theorem \ref{thm:general-rl})}
Lemma~\ref{lem:gen_infinite-rl} implies that as soon as we reach epoch $i^*$, \texttt{ARL} identifies the model class with high probability, i.e., for each $i \ge i^*$, we have $m^{(i)}=m^*$. However, before that, we do not have any guarantee on the regret performance of \texttt{ARL}. Since at every episode the regret can be at most $H$, 
the cumulative regret up until the $i^*$ epoch is upper bounded by $\tau_{i^*-1}H$, 
which is at most $\mathcal{O} \left(\log K \max \left \lbrace \frac{H^4  \log(1/\delta)}{(\frac{\Delta^2}{2}-2H\eta)^2}, H^2 \log\left(\frac{\cN(\cP_M)}{\delta}\right) \right \rbrace \right)$, if the model classes are infinite, and $\mathcal{O} \left(\log K \max \left \lbrace \frac{H^4  \log(1/\delta)}{(\frac{\Delta^2}{2}-2H\eta)^2}, H^2 \log\left(\frac{|\cP_M|}{\delta}\right) \right \rbrace \right)$, if the model classes are finite. Note that this is the cost we pay for model selection. 

Now, let us bound the regret of \texttt{ARL} from epoch $i^*$ onward. 
Let $R^{\texttt{UCRL-VTR}}(k_i,\delta_i,\cP_{m^{(i)}})$ denote the cumulative regret of \texttt{UCRL-VTR}, when it is run for $k_i$ episodes with confidence level $\delta_i$ for the family $\cP_{m^{(i)}}$. Now, using the result of \cite{ayoub2020model}, we have
\begin{align*}
    R^{\texttt{UCRL-VTR}}(k_i,\delta_i,\cP_m^{(i)}) &\le 1+H^2\dim_{\cE}\left(\cM_{m^{(i)}},\frac{1}{k_iH}\right) \\&+4\sqrt{\beta_{k_i}(\cP_{m^{(i)}},\delta_i)\dim_{\cE}\left(\cM_{m^{(i)}},\frac{1}{k_iH}\right)k_i H}+H\sqrt{2k_iH\log(1/\delta_i)}
\end{align*}
with probability at least $1-2\delta_i$.
With this and Lemma \ref{lem:gen_infinite-rl}, the regret of \texttt{ARL} after $K$ episodes (i.e., after $T=KH$ timesteps) is given by
\begin{align*}
    R(T) &\leq \tau_{i^*-1}H + \sum_{i=i^*}^N R^{\texttt{UCRL-VTR}}(k_i,\delta_i,\cP_{m^{(i)}}) \\
     & \leq \tau_{i^*-1}H + N+ \sum_{i=i^*}^N H^2\dim_{\cE}\left(\cM_{m^*},\frac{1}{k_iH}\right)\\& +4 \sum_{i=i^*}^N \sqrt{\beta_{k_i}(\cP_{m^*},\delta_i)\dim_{\cE}\left(\cM_{m^*},\frac{1}{k_iH}\right)k_i H}+ \sum_{i=i^*}^N H\sqrt{2k_iH\log(1/\delta_i)}~.
\end{align*}
The above expression holds with probability at least $1-c'M\delta-2\sum_{i=i^*}^{N}\delta_i$ for infinite model classes, and with probability at least $1-cM\delta-2\sum_{i=i^*}^{N}\delta_i$ for finite model classes. Let us now compute the last term in the above expression. Substituting $\delta_i = \delta/2^i$, we have
\begin{align*}
    \sum_{i=i^*}^N H\sqrt{2k_iH\log(1/\delta_i)} &= \sum_{i=i^*}^N H\sqrt{2k_i H \,i \,\log(2/\delta)} \\
     & \leq H\sqrt{2HN\log(2/\delta)} \sum_{i=1}^N \sqrt{k_i} \\
     &= \mathcal{O}\left(  H \sqrt{KH N \log(1/\delta)}\right) =  \mathcal{O}\left( H \sqrt{T\log K \log(1/\delta)}\right),
\end{align*}
where we have used that the total number of epochs $N=\cO(\log K)$, and that
\begin{align*}
    \sum_{i=1}^N \sqrt{k_i} & = \sqrt{k_N}\left(1 + \frac{1}{\sqrt{2}} + \frac{1}{2} + \ldots N\text{-th term} \right)\\
    & \leq \sqrt{k_N}\left(1 + \frac{1}{\sqrt{2}} + \frac{1}{2} + ... \right)= \frac{\sqrt{2}}{\sqrt{2} -1} \sqrt{k_N}  \leq \frac{\sqrt{2}}{\sqrt{2} -1} \sqrt{K}.
    \end{align*}
     Next, we can upper bound the third to last term in the regret expression as
  \begin{align*}
      \sum_{i=i^*}^N H^2\dim_{\cE}\left(\cM_{m^*},\frac{1}{k_iH}\right) \le H^2 N \dim_{\cE}\left(\cM_{m^*},\frac{1}{KH}\right) = \cO \left(H^2d_{\cE}(\cM^*) \log K \right).
  \end{align*}
    Now, notice that, by substituting $\delta_i = \delta/2^i$, we can upper bound $\beta_{k_i}(\cP_{m^*},\delta_i)$ as follows:
  \begin{align*}
      \beta_{k_i}(\cP_{m^*},\delta_i) &=  \mathcal{O}\left( H^2 \, i \, \log \left( \frac{\cN(\mathcal{P}_{m^*},\frac{1}{k_i H},\norm{\cdot}_{\infty,1})}{\delta}\right) + H^2 \left(1+ \sqrt{i\log \frac{ k_iH(
    k_i H+1)}{\delta}} \right) \right)\\ & \leq \mathcal{O}\left( H^2 \, N \, \log \left( \frac{\cN(\mathcal{P}_{m^*},\frac{1}{K H},\norm{\cdot}_{\infty,1})}{\delta}\right) + H^2 \left(1+ \sqrt{N \log \frac{K H}{\delta}} \right) \right) \\
      & \leq \mathcal{O}\left( H^2 \, \log K \, \log \left( \frac{\cN(\mathcal{P}_{m^*},\frac{1}{K H},\norm{\cdot}_{\infty,1})}{\delta}\right) + H^2  \sqrt{ \log K \log (K H/\delta)}  \right)
  \end{align*}
  for infinite model classes, and $\beta_{k_i}(\cP_{m^*},\delta_i) = \cO\left(H^2\log K\log\left(\frac{|\cP_{m^*}|}{\delta} \right) \right)$ for finite model classes.
  With this, the second to last term in the regret expression can be upper bounded as 
  \begin{align*}
    & \sum_{i=i^*}^N 4\sqrt{\beta_{k_i}(\cP_{m^*},\delta_i)\dim_{\cE}\left(\cM_{d_i},\frac{1}{k_iH}\right)k_i H}  \\
    & \leq \mathcal{O} \bigg(H\sqrt{  \log K \, \log \left( \frac{\cN(\mathcal{P}_{m^*},\frac{1}{K H},\norm{\cdot}_{\infty,1})}{\delta}\right) + \sqrt{ \log K \log (K H/\delta)}  } \\ 
    & \qquad \times \sqrt{\dim_{\cE}\left(\cM_{m^*},\frac{1}{K H}\right)} \sum_{i=i^*}^N \sqrt{k_i H} \bigg)\\
    & \leq \mathcal{O} \bigg(H\sqrt{ \log K \, \log \left( \frac{\cN(\mathcal{P}_{m^*},\frac{1}{K H},\norm{\cdot}_{\infty,1})}{\delta}\right) \log\left(\frac{KH}{\delta}\right)} \\ 
    & \qquad \times \sqrt{KH\,\,\dim_{\cE}\left(\cM_{m^*},\frac{1}{K H}\right)}\bigg)\\
    & = \cO \left(H\sqrt{Td_{\cE}(\cM^*)\log(\cN(\cP^*)/\delta)\log K \log(T/\delta)} \right)
  \end{align*}
  for infinite model classes. Similarly, for finite model classes, we can upper bound this term by $\cO \left(H\sqrt{Td_{\cE}(\cM^*)\log\left(\frac{|\cP^*|}{\delta} \right)\log K} \right)$.
 
 Hence, for infinite model classes, the final regret bound can be written as
\begin{align*}
   R(T) &= \mathcal{O} \left(\log K \max \left \lbrace \frac{H^4  \log(1/\delta)}{(\frac{\Delta^2}{2}-2H\eta)^2}, H^2 \log\left(\frac{\cN(\cP_M)}{\delta}\right) \right \rbrace \right)\\ &\quad + \cO \left(H^2d_{\cE}(\cM^*) \log K+H\sqrt{Td_{\cE}(\cM^*)\log(\cN(\cP^*)/\delta)\log K \log(T/\delta)} \right)~.
\end{align*}
 The above regret bound holds with probability greater than
 \begin{align*}
      1-c'M\delta -\sum_{i=i^*}^{N}\frac{\delta}{2^{i-1}} \ge 1-c'M\delta -\sum_{i\ge 1}\frac{\delta}{2^{i-1}} =1-c'M\delta-2\delta~,
 \end{align*}
 which completes the proof of Theorem \ref{thm:general-rl}.
 
Similarly, for finite model classes, the final regret bound can be written as
\begin{align*}
    R(T) &= \mathcal{O} \left(\log K \max \left \lbrace \frac{H^4  \log(1/\delta)}{(\frac{\Delta^2}{2}-2H\eta)^2}, H^2 \log\left(\frac{|\cP_M|}{\delta}\right) \right \rbrace \right)\\ & \quad+  \cO \left(H^2d_{\cE}(\cM^*) \log K+H\sqrt{Td_{\cE}(\cM^*)\log\left(\frac{|\cP^*|}{\delta} \right)\log K} \right),
\end{align*}
which holds with probability greater than $1-cM\delta-2\delta$.
\end{appendix}
\end{document}